\documentclass{article}
\usepackage{iclr2027_conference,times}

\usepackage{hyperref}
\usepackage{url}
\usepackage{amsmath,amssymb,amsthm}
\usepackage{booktabs}
\usepackage{graphicx}
\usepackage{multirow}

\newcommand{\pt}{p_{\mathrm{T}}}
\newcommand{\ps}{p_{\phi}}
\newcommand{\pb}{p_{\mathrm{base}}}
\newcommand{\aeff}{\alpha_{\mathrm{eff}}}
\newtheorem{theorem}{Theorem}

\title{Sharpen Without Search: On-Policy Distillation of Sequence-Level Power Distribution}

\author{
Erfan Baghaei Potraghloo$^{u\dagger}$ \thanks{Corresponding author: \texttt{baghaeip@usc.edu}.} \; , Seyedarmin Azizi$^{u\dagger}$,
Arya Fayyazi$^{u}$, Saeid Shokoufa$^{u}$,\\
\textbf{Mehdi Kamal$^{u}$, Souvik Kundu$^{i}$, Massoud Pedram$^{u}$}\\
$^{u}$University of Southern California, Los Angeles, USA\\
$^{i}$Intel AI, USA\\
$\dagger$ Equal contribution authors.
}

\iclrfinalcopy

\begin{document}
\maketitle
\lhead{Preprint}

\begin{abstract}

A language model can give a correct answer more probability than any single
incorrect answer and still usually sample an incorrect one, because the incorrect
answers together hold more probability. The \emph{power distribution} raises each
complete answer's probability to a power above one and renormalizes, shifting
probability toward answers the model finds most likely (sharpening). Sampling from
it improves reasoning without changing parameters, but needs many scored candidates
per query. We show that a model can instead be trained to produce such answers in
one generation. On-policy power distillation (\textsc{oppd}) runs a sequential Monte
Carlo sampler in which the model being trained generates candidates and a frozen
teacher's power distribution weights them; the same probabilities weight each answer
in a maximum-likelihood update. Training raises single-generation accuracy by up to
$23.0$ points on MATH500 and $27.3$ on GSM8K over the untrained model at the same
temperature, and one generation scores $2.4$ and $3.5$ points above published power
sampling with $64$ candidates, recovering $94\%$ of the gain that $16$ candidates
give the untrained model. For context, against GRPO trained with verified rewards
from the same checkpoint and budget, \textsc{oppd} scores $3.8$, $4.0$ and $5.4$
points higher on MATH500, GSM8K and AIME using no reference answers; the two are
complementary, and \textsc{oppd} applied after GRPO adds up to $9.3$ points.
Trained only on mathematics, \textsc{oppd} raises HumanEval accuracy by up to $5.3$
points. One loss coefficient moves the sharpening exponent the model absorbs between
$1.19$ and $2.02$, against $1.14$ for ordinary on-policy distillation, and it rises
mostly on the model's own answers. Gains hold across model families and sizes,
including a model already trained with verified rewards, where lowering the
temperature gives nothing and \textsc{oppd} adds $4.4$ points on MATH500. Code:
\url{https://github.com/ArminAzizi98/OPPD}.
\end{abstract}

\begin{figure}[!t]
\centering
\includegraphics[width=\textwidth]{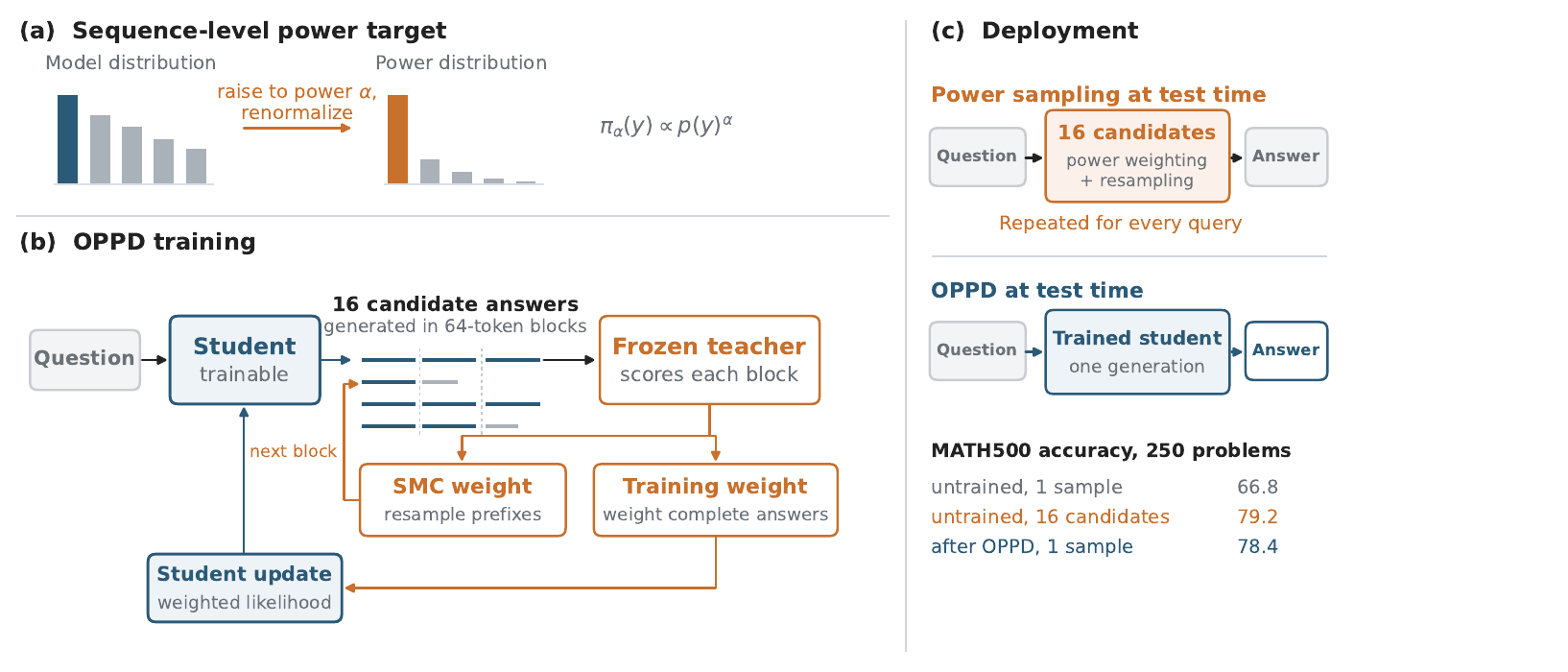}
\caption{\textsc{Oppd} moves power sampling from inference time into training.
(a) The power distribution gives each complete answer a probability proportional to
the model's probability of that answer raised to the power $\alpha$. (b) One
training step: (1) the trainable student generates its own candidate answers, (2) the
candidates grow in blocks of tokens, and (3) a frozen teacher scores each new block.
(4) The teacher probabilities give two weights. The importance weight (``SMC
weight'') decides which partial answers are discarded or duplicated, and the training
weight sets how much each complete answer contributes to the student update.
(c) Power
sampling generates and scores $16$ candidates for every query, while the trained
student generates one answer and obtains $94\%$ of the sampler's gain on MATH500
(Table~\ref{tab:amort}).}
\label{fig:method}
\end{figure}

\section{Introduction}

Let $p$ be an autoregressive language model that assigns probability
$p(y \mid x) = \prod_t p(y_t \mid x, y_{<t})$ to an answer $y$ given a prompt
$x$, where the \emph{prefix} $y_{<t}$ denotes the tokens before position $t$.
Sampling from $p$ can return an incorrect answer even when the correct answer is
more probable than any single incorrect one, because the incorrect answers together
hold more probability. The \emph{power distribution} with exponent $\alpha \ge 1$,
\begin{equation}
\pi_\alpha(y \mid x) = \frac{p(y \mid x)^{\alpha}}{Z_\alpha(x)},
\qquad Z_\alpha(x) = \textstyle\sum_{y'} p(y' \mid x)^{\alpha},
\label{eq:power}
\end{equation}
where $Z_\alpha(x)$ is the normalizing constant, moves probability toward the most
probable answers, an effect we call \emph{sharpening}. Under $\pi_\alpha$, the ratio
between the probabilities of two answers is their ratio under $p$ raised to the power
$\alpha$. Drawing samples from $\pi_\alpha$
instead of from $p$ raises reasoning accuracy by amounts comparable to
reinforcement-learning post-training, without changing any parameter
\citep{karan2026reasoning,ji2026scalable}. On $250$ MATH500 problems,
\texttt{Qwen2.5-Math-7B} scores $66.8\%$ when it generates one answer and $79.2\%$
when a sampler with sixteen candidates draws approximate samples from $\pi_\alpha$.
This sampler, however, has to run again for every query.

In this paper we ask whether this cost can be paid once, during training. Ordinary
distillation trains a student to match a teacher's next-token probabilities. The
next-token probabilities of $\pi_\alpha$, however, depend on all possible
continuations of the prefix, so no teacher forward pass computes them. Samples from
$\pi_\alpha$ can therefore only be drawn approximately, by a sampler run at
inference time. Our method, on-policy power distillation (\textsc{oppd}), trains the
student on the output of such a sampler, the sequential Monte Carlo (SMC) sampler of
\citet{powersmc2026}. This sampler extends $N$ candidate answers in blocks of tokens and gives each
candidate an importance weight, which measures how much more probable it is under
$\pi_\alpha$ than under the distribution that generated it. When these weights become
too uneven, it replaces low-weight candidates with copies of high-weight ones. In \textsc{oppd} the student generates the candidates and a frozen
teacher defines $\pi_\alpha$; in self-distillation, the teacher is a copy of the
initial student. The teacher probabilities that set the importance weights also give
each complete answer a training weight, and the student is updated by weighted
maximum likelihood on its own answers (Figure~\ref{fig:method}). A second loss term,
the token-level anchor with coefficient $\lambda$, keeps the student's next-token
distributions near those of the teacher at temperature $1/\alpha$.

\paragraph{Contributions.} (1) A training method that uses no reference answers,
verifier or reward model, in which one teacher forward pass per block of tokens
supplies both the importance weights and the training weights. (2) A trained model that replaces the
sampler, scoring above published power sampling with $64$ candidates on the same
model and obtaining $94\%$ of the gain that sixteen candidates give the untrained
model. (3) A reference point for the size of this gain. Group
Relative Policy Optimization (GRPO) trains on verified rewards computed from
reference answers, so it uses information \textsc{oppd} never reads. Trained from
the same checkpoint for the same number of steps, \textsc{oppd} scores higher on
MATH500, GSM8K and AIME, and we trace the difference to training groups in which
every sampled answer receives the same reward and GRPO gets no learning signal.
Applied after GRPO, \textsc{oppd} improves the model further, so the two are
complementary. We are not proposing \textsc{oppd} as a replacement for
reinforcement learning from verified rewards, and we compare against one standard
implementation at one budget rather than against that literature. (4) A measurement of how much
sharpening the student absorbs, as an exponent fitted between the student's token
probabilities and those of its initialization (Section~\ref{sec:aeff}). This
exponent decreases monotonically as $\lambda$ increases and rises mainly on the
answers the student writes itself. A quantity measured before training predicts the
best $\lambda$ for each model.

\section{Background}
\label{sec:background}

\subsection{Sharpening complete answers versus sharpening tokens}
\label{sec:whytemp}

Lowering the sampling temperature is cheaper and may appear to have the same
effect, but the two operations differ. Sampling at temperature $\tau$ is
\emph{token-level}: it replaces each next-token distribution by one proportional to
its probabilities raised to the power $1/\tau$, renormalizing at every position. The
power distribution is \emph{sequence-level}: \eqref{eq:power} renormalizes once,
over complete answers. Temperature therefore compares next tokens without
information about the answers they lead to, while the power distribution gives
little total probability to answers that begin with a prefix whose continuations are
all uncertain. In the two-token example of Appendix~\ref{app:toy}, the power
distribution at $\alpha=2$ places $0.213$ on answers beginning with the uncertain
token and temperature $1/2$ places $0.308$ there.

\subsection{Drawing samples from the power distribution}
\label{sec:smc}

The constant $Z_\alpha$ in \eqref{eq:power} sums over every possible answer and
cannot be computed, so samples from $\pi_\alpha$ are drawn approximately with
sequential Monte Carlo \citep{doucet2001smc,delmoral2004feynman,delmoral2006smc}.
The sampler extends $N$ candidate answers in parallel, which the SMC literature calls
particles. The next token of each candidate is drawn from a \emph{proposal
distribution} $q_t$, any distribution that depends only on the prompt and the
prefix. With $\pt$ the model whose power distribution is sampled (the teacher in
\textsc{oppd}), each candidate carries an \emph{importance weight}, the product of
the \emph{incremental importance weights}
\begin{equation}
\omega_t = \frac{\pt(y_t \mid x, y_{<t})^{\alpha}}{q_t(y_t \mid x, y_{<t})},
\label{eq:incw}
\end{equation}
over its tokens so far. Candidate $i$ thus has importance weight
$\tilde W^{(i)}_t = \prod_{s \le t} \omega^{(i)}_s$ after $t$ tokens. Rescaling
these weights to sum to one gives the \emph{normalized importance weights}
$W^{(i)}$. The \emph{effective sample size}, $\mathrm{ESS} = (\sum_i
(W^{(i)})^2)^{-1}$, approximates the number of candidates with non-negligible weight:
it equals $N$ when all weights are equal and $1$ when one candidate holds all of the
weight. When the ESS falls below a threshold, the sampler \emph{resamples} the
candidates: it discards low-weight candidates, duplicates high-weight ones and
resets all weights to be equal. Resampling moves computation toward high-weight
prefixes while the answers are still being written (Appendix~\ref{app:smc}).

\paragraph{The effective sample size measures a divergence.} The \emph{R\'enyi
divergence} of order $\beta$ between distributions $P$ and $Q$ is
\begin{equation}
D_\beta(P \Vert Q) \;=\; \frac{1}{\beta - 1}\,
\log \sum_y P(y)^{\beta}\, Q(y)^{1-\beta},
\qquad \beta > 0,\; \beta \neq 1,
\label{eq:renyidiv}
\end{equation}
with $\beta = 1$ defined by continuity as the Kullback--Leibler (KL) divergence
\citep{renyi1961measures}. Let $q$ be the proposal probability of a complete answer
and $w = \pi_\alpha/q$, its importance weight divided by $Z_\alpha$. Then
$\mathbb{E}_q[w]=1$ and
$\mathbb{E}_q[w^2] = e^{D_2(\pi_\alpha \Vert q)}$, so for large $N$
\begin{equation}
\mathrm{ESS}/N \;\approx\; \mathbb{E}_q[w]^2/\mathbb{E}_q[w^2]
\;=\; e^{-D_2(\pi_\alpha \Vert q)} .
\label{eq:exchange}
\end{equation}
The ESS measured while the sampler runs therefore estimates a divergence between
$\pi_\alpha$ and the proposal distribution, which we use in
Section~\ref{sec:mechmain}.

\begin{theorem}[Best proposal distribution that depends only on the prefix]
\label{thm:opt}
Consider all proposal distributions that are positive wherever $\pt$ is. The
unique choice that minimizes the conditional variance of the incremental
importance weight \eqref{eq:incw} is to sample from the teacher at temperature
$1/\alpha$, that is $q^{\star}_t(v) \propto \pt(v \mid x, y_{<t})^{\alpha}$.
Under this choice $\omega_t$ equals $Z_t = \sum_u \pt(u \mid x, y_{<t})^{\alpha}$
whatever token is drawn.
\end{theorem}

This is the locally optimal proposal distribution of \citet{doucet2000smc}, written
for the power distribution; the proof is in Appendix~\ref{app:proofs}. The R\'enyi
entropy
$H_\alpha(p) = \frac{1}{1-\alpha}\log\sum_v p(v)^{\alpha}$ is small when $p$ is
concentrated, and $\log Z_t = (1-\alpha)H_\alpha(\pt(\cdot\mid x,y_{<t}))$, so
summing over positions gives the log importance weight of an answer under $q^\star$,
\begin{equation}
\log \tilde W = (1-\alpha)\textstyle\sum_t H_\alpha\!\left(\pt(\cdot \mid x, y_{<t})\right),
\label{eq:renyi}
\end{equation}
which is non-positive because $\alpha>1$. A candidate therefore keeps its weight
when the teacher was confident at every position of its answer and loses weight
where the teacher was uncertain. This is the sampler's only selection rule, and it
uses only the teacher's output probabilities.

As usual in SMC, the sampler does not estimate the twist function $V_t =
\sum_{y_{>t}} \pt(y_{>t}\mid y_{1:t})^{\alpha}$, on which the next-token distribution
of $\pi_\alpha$ depends. Weighting and resampling correct for it.

\section{Method}
\label{sec:method}

We want the student $\ps$, with parameters $\phi$, to minimize
$D_\beta(\pi_\alpha \Vert \ps)$ for some $\beta > 0$. Because $Z_\alpha$ cannot be computed, $\pi_\alpha$ can be neither
evaluated nor sampled exactly, and the gradient of the divergence has both
problems. Section~\ref{sec:grad} rewrites the gradient to avoid both, and
Sections~\ref{sec:where}--\ref{sec:loss} supply the samples, the proposal and the
loss, following steps (1) to (4) of Figure~\ref{fig:method}b.

\subsection{What the gradient requires}
\label{sec:grad}

\begin{theorem}[Gradient of the R\'enyi divergence]
\label{thm:master}
For every $\beta>0$,
$\nabla_\phi D_\beta(\pi_\alpha \Vert \ps)
 = -\,\mathbb{E}_{\bar\pi^{(\beta)}}\!\left[\nabla_\phi \log \ps(y)\right]$,
where $\bar\pi^{(\beta)} \propto \pt^{\alpha\beta}\,\ps^{\,1-\beta}$.
\end{theorem}

The factor $\tfrac{1}{\beta-1}$ in \eqref{eq:renyidiv} cancels the factor
$(1-\beta)$ from the chain rule for every $\beta$, and $Z_\alpha$, which does not
depend on $\phi$, drops out of the gradient although it appears in the objective
(Appendix~\ref{app:proofs}). One estimator therefore serves every $\beta$, which
makes $\beta$ a free parameter. In practice, answers can be drawn from any $q$
positive wherever $\bar\pi^{(\beta)}$ is, weighted by $\bar\pi^{(\beta)}(y)/q(y)$,
and used in weighted maximum likelihood, the estimator of reweighted wake-sleep
\citep{bornschein2015rws} and of Markovian score climbing \citep{naesseth2020msc}
applied here to the R\'enyi family. Holding the weights constant during differentiation is
exact, since the derivation accounts for how $\bar\pi^{(\beta)}$ depends on $\phi$
(Appendices~\ref{app:derivation} and~\ref{app:verify}).

\subsection{Where the answers come from}
\label{sec:where}

The weight $\bar\pi^{(\beta)}(y)/q(y) \propto
\pt(y)^{\alpha\beta}\ps(y)^{1-\beta}/q(y)$ needs three numbers at every token: the
teacher's $\log\pt(y_t)$, the student's $\log\ps(y_t)$ and the proposal's
$\log q_t(y_t)$. When the student is the proposal distribution, the sampler of
Section~\ref{sec:smc} already computes all three: \eqref{eq:incw} requires
$\log\pt(y_t)$ and $\log q_t(y_t)$, and the student computes $\log\ps(y_t)$ while
generating the token. This choice adds no model evaluations, and one
teacher forward pass per block supplies both the importance weights that decide
resampling and the weights of Theorem~\ref{thm:master} (step (4) of
Figure~\ref{fig:method}b). Because the weighted answers are its own, the student is trained on the prefixes its
own errors lead to \citep{kim2016sequence,agarwal2024gkd,gu2024minillm}.

\subsection{Choosing the proposal temperature}
\label{sec:temp}

The proposal distribution can be the student as it stands, $q_t = \ps$, or the
student at temperature $1/\alpha$, $q_t \propto \ps^{\alpha}$. With $q_t = \ps$,
\begin{equation}
\log \omega_t = \underbrace{\alpha\!\left(\log\pt(y_t) - \log\ps(y_t)\right)}_{\text{teacher--student difference}}
\;+\; \underbrace{(\alpha-1)\log\ps(y_t)}_{\text{student's log-probability of its own draw}} .
\label{eq:decomp}
\end{equation}
Only the first term carries information useful for training. The second depends
only on which token the student happened to sample, so its distribution is the same
for good and bad answers. Its standard deviation is $4.6$ nats (natural-log units)
per token at $\alpha = 4$. Over a block of $\Delta = 64$ tokens, the log
importance weights of two candidates then differ by about $4.6\sqrt{64} \approx 37$
nats for reasons unrelated to their content. One candidate then receives all of the
normalized weight, and the measured $\mathrm{ESS}/N$ stays at $1/N = 0.0625$ from the
first block onward
(Appendix~\ref{app:proposal}). The proposal $q_t \propto \ps^{\alpha}$ is the
variance-optimal choice of Theorem~\ref{thm:opt} with the student in place of the
teacher, and with it $\log q_t(y_t) = \alpha\log\ps(y_t) - \log Z^{\phi}_t$ with
$Z^{\phi}_t = \sum_u \ps(u \mid x, y_{<t})^{\alpha}$, and
$\log\omega_t = \alpha(\log\pt(y_t) - \log\ps(y_t)) + \log Z^{\phi}_t$. The second
term of \eqref{eq:decomp} disappears, and $\log Z^{\phi}_t$ depends on the prefix but
not on the token drawn, so it adds no sampling noise. With this proposal, the
measured $\mathrm{ESS}/N$ during training lies between $0.39$ and $0.64$ for every
model we train.

\subsection{The loss}
\label{sec:loss}

Theorem~\ref{thm:master} averages over $\bar\pi^{(\beta)}$, but the candidates
are drawn from $q$. Self-normalized importance sampling corrects for this mismatch
by assigning candidate $i$ the \emph{training weight}
\begin{equation}
\tilde A^{(i)} \;=\; \frac{\bar\pi^{(\beta)}\big(y^{(i)}\big)}{q\big(y^{(i)}\big)},
\qquad
\bar A^{(i)} \;=\; \frac{\tilde A^{(i)}}{\sum_{j=1}^{N} \tilde A^{(j)}},
\qquad
\sum_{i=1}^{N} \bar A^{(i)} = 1 .
\label{eq:trainweight}
\end{equation}
Both $\bar\pi^{(\beta)}$ and $q$ factor over tokens, and their products over
hundreds of tokens underflow, so we accumulate the logarithm one token at a time:
\begin{equation}
\log \tilde A^{(i)} = \sum_t \Big[\alpha\beta \log \pt(y^{(i)}_t \mid \cdot)
+ (1-\beta)\log \ps(y^{(i)}_t \mid \cdot) - \log q_t(y^{(i)}_t \mid \cdot)\Big] ,
\label{eq:lossweight}
\end{equation}
where each term is conditioned on $x$ and $y^{(i)}_{<t}$, and $\ps$ is the value
recorded when the token was drawn. Any constant common to all candidates cancels in
the normalization. The training weight is distinct from the importance weight of
\eqref{eq:incw}, which decides when the sampler resamples the candidates. Their
per-token increments agree at $\beta=1$, but the sum in \eqref{eq:lossweight}
restarts at each block boundary while the importance weight accumulates from the
start of the answer.

Replacing the expectation in Theorem~\ref{thm:master} by the weighted average over
the $N$ candidates gives a consistent gradient estimate with bias of order $1/N$.
Holding $\bar A^{(i)}$ fixed during differentiation makes the estimate the gradient
of a loss, and dividing the loss by the mean answer length $\bar T = \frac1N \sum_i T_i$
keeps its scale independent of answer length (Appendix~\ref{app:derivation}).
Writing $\pi^{(\alpha)}_{\mathrm{T}} \propto \pt^{\alpha}$ for the teacher's
next-token distribution at temperature $1/\alpha$, the objective is
\begin{equation}
\mathcal{L}(\phi) =
\underbrace{-\frac{1}{\bar T}\sum_{i}\big[\bar A^{(i)}\big]_{\mathrm{sg}}
\sum_{t} \log \ps\!\left(y^{(i)}_t \mid x, y^{(i)}_{<t}\right)}_{\text{sequence-level term}}
\;+\;\lambda\,
\underbrace{\frac{1}{\bar T N}\sum_{i,t}
\mathrm{KL}\!\left(\pi^{(\alpha)}_{\mathrm{T}} \,\Vert\, \ps\right)}_{\text{token-level anchor}} ,
\label{eq:loss}
\end{equation}
where $[\cdot]_{\mathrm{sg}}$ is held fixed during differentiation and both KL
arguments are conditioned on $x$ and $y^{(i)}_{<t}$.

The sequence-level term is the core of the method: its optimum is $\pi_\alpha$, it is
the only term that uses information about complete answers, and it is the only term
left when $\lambda = 0$. The token-level anchor does not follow from
Theorem~\ref{thm:master} and moves the attainable optimum toward the teacher at
temperature $1/\alpha$, the token-level sharpening of Section~\ref{sec:whytemp}. We
include it because this forward KL is mass-covering: it penalizes the student for
giving little probability to tokens that $\pi^{(\alpha)}_{\mathrm{T}}$ considers
likely. It therefore resists the concentration of all probability onto one answer,
a failure we observe on models whose outputs already have little diversity. The
coefficient $\lambda$ trades closeness to $\pi_\alpha$ against this protection
(Section~\ref{sec:mechmain}). The training weights multiply a log-likelihood rather
than a divergence, because a weighted token-level divergence cannot reach
$\pi_\alpha$ (Appendix~\ref{app:wrongloss}).

\subsection{Implementation}
\label{sec:impl}

The student generates $N = 16$ candidates at temperature $1/\alpha$ in blocks of
$\Delta = 64$ tokens. The teacher scores all $N\Delta$ new tokens of a block in one
forward pass, which reduces the sequential teacher calls for a $T$-token answer
from $T$ to $T/\Delta$. Three details affect correctness: the importance and
training weights use separate accumulators; recorded training weights are
renormalized after every resampling step, so that duplicated candidates are not
counted more than once (Appendix~\ref{app:resample}); and finished candidates are
charged the mean per-token decrease in log importance weight of unfinished ones, so
that stopping early does not raise a candidate's weight
(Appendix~\ref{app:brevity}).

\section{Measuring absorbed sharpening}
\label{sec:aeff}

Higher accuracy alone does not show that the student moved toward $\pi_\alpha$,
so we test a consequence that is specific to this movement. Let $\pb$ be the model
the student was initialized from, and let a \emph{context} $c$ be a prompt with a
prefix. If training sharpened the student relative to $\pb$ with some exponent $a$,
so that $\ps \approx \pb^{\,a}/Z$, then
\begin{equation}
\log \ps(v \mid c) = a\,\log \pb(v \mid c) - \log Z(c),
\label{eq:aeff}
\end{equation}
which is a line of slope $a$ in $\log \pb(v \mid c)$ at each context $c$. A slope of
$1$ means that no sharpening was absorbed, and a slope of $\alpha$ means that the
student equals $\pb$ at temperature $1/\alpha$. We estimate the slope by least
squares of $\log\ps$ on $\log\pb$ over the $200$ tokens most likely under $\pb$,
weighted by $\pb(v \mid c)$ so that the near-zero tail does not dominate. We
report the median slope across contexts as the \emph{absorbed exponent}
$\aeff(\pb \to \ps)$. In self-distillation $\aeff(\pb \to \pt) = 1.000$ exactly, so
no measured increase can come from copying a sharper model; with a larger teacher,
the control is ordinary on-policy distillation against the same teacher. The slope
is dominated by noise when $\pb$ is close to deterministic, so we report $\aeff$
only where the entropy of $\pb$ is large enough for the fit to separate the models,
and a direct entropy measurement elsewhere. Appendix~\ref{app:fullmech} gives both
and validates the estimator on synthetic logits.

\section{Experiments}
\label{sec:exp}

\paragraph{Setup.} Our main experiments use \texttt{Qwen2.5-Math-7B}
\citep{yang2024qwen25} in self-distillation, which makes the control of
Section~\ref{sec:aeff} exact. We repeat the training on \texttt{Qwen3-8B-Base}
\citep{yang2025qwen3}, \texttt{Llama-3.1-8B-Instruct} \citep{dubey2024llama3},
\texttt{Qwen3-14B-Base} and \texttt{deepseek-math-7b-rl}
\citep{shao2024deepseekmath}, and use \texttt{Qwen2.5-Math-1.5B} as the student for the
sweep over $\lambda$. Training prompts are GSM8K questions \citep{cobbe2021gsm8k}
except in Appendix~\ref{sec:openr1}, and no reference answer is used in training.
We set $\alpha=4$, increased linearly from $1$ over the first $20\%$ of steps, with
$N=16$ candidates, blocks of $\Delta=64$ tokens, resampling when the ESS falls below
$N/2$, and $\beta=0.7$. Training runs for $300$ steps at two prompts per step
(Appendix~\ref{app:setup}). We generate four answers per problem for $200$ problems
each from GSM8K and MATH500 \citep{lightman2024lets}, graded with the reference
power-sampling implementation. We do the same for AIME $2024$ and $2025$ pooled ($60$
problems, $3500$-token limit; $4096$ for \texttt{Qwen3-14B-Base}). \emph{pass@1} is the
fraction of the four answers that are correct, averaged over problems, and
\emph{best-of-4} counts a problem solved if any answer is correct. The paired
standard error on a pass@1 difference is about $2.9$ points. Our baseline is ordinary
on-policy distillation against the teacher at temperature $1/\alpha$, which is
\eqref{eq:loss} with the anchor alone and no training weights or resampling. Our
measurements of the untrained model are within $3$ points of published numbers, in
both directions (Appendix~\ref{app:harness}). \texttt{Qwen2.5-Math-7B} improves under
many training signals, including rewards uncorrelated with correctness
\citep{shao2025spurious}, so the replications on other families in
Section~\ref{sec:replmain} carry as much of the claim as this model does.

\subsection{Accuracy}
\label{sec:acc}

\begin{table}[t]
\caption{Self-distillation on \texttt{Qwen2.5-Math-7B}. The published rows are
for the same model. Generating a single answer, the model trained with \textsc{oppd}
scores above power sampling with $64$ candidates and matches GRPO, which uses
reference answers during training.}
\label{tab:main}
\centering\small
\begin{tabular}{lcccc}
\toprule
 & \multicolumn{2}{c}{MATH500} & \multicolumn{2}{c}{GSM8K} \\
\cmidrule(lr){2-3}\cmidrule(lr){4-5}
 & pass@1 & best-of-4 & pass@1 & best-of-4 \\
\midrule
\multicolumn{5}{l}{\emph{sampling at temperature $1$}}\\
untrained & $50.25$ & $82.0$ & $59.25$ & $91.0$ \\
on-policy distillation (baseline) & $62.25$ & $86.5$ & $74.88$ & $93.5$ \\
\textsc{oppd}, $\lambda=1$ & $64.00$ & $84.0$ & $79.38$ & $96.0$ \\
\textsc{oppd}, $\lambda=0$ & $73.25$ & $88.5$ & $86.50$ & $95.5$ \\
\midrule
\multicolumn{5}{l}{\emph{sampling at temperature $1/\alpha$}}\\
untrained & $70.13$ & $85.0$ & $76.13$ & $92.5$ \\
on-policy distillation (baseline) & $68.25$ & $84.0$ & $74.38$ & $90.0$ \\
\textsc{oppd}, $\lambda=1$ & $72.88$ & $83.5$ & $86.88$ & $93.0$ \\
\textsc{oppd}, $\lambda=0$ & $\mathbf{78.50}$ & $86.0$ & $\mathbf{91.13}$ & $95.5$ \\
\midrule
\multicolumn{5}{l}{\emph{published for this model}}\\
power sampling, $64$ candidates \citep{powersmc2026} & $76.1$ & --- & $87.6$ & --- \\
GRPO \citep{powersmc2026} & $78.0$ & --- & --- & --- \\
\bottomrule
\end{tabular}
\end{table}

Generating one answer at temperature $1/\alpha$, the model trained with
\textsc{oppd} scores $2.4$ points above published power sampling with $64$
candidates on MATH500 and $3.5$ points above it on GSM8K (Table~\ref{tab:main}). Two further seeds score $79.50$ and $81.25$ on MATH500, so every run exceeds
published power sampling (Appendix~\ref{sec:seeds}). At temperature $1$, this model gains $23.0$
and $27.3$ points over the untrained model and
$11.0$ and $11.6$ points over ordinary on-policy distillation. Best-of-4 also rises on both benchmarks, so the model solves
more problems rather than repeating the same answer. Lowering the temperature to $1/\alpha$ requires no
training, so comparing models at that temperature separates the contribution of
training from that of decoding. At this temperature, ordinary on-policy distillation
loses $1.9$ and $1.8$ points relative to the untrained model, both within one
standard error, while \textsc{oppd} gains $8.4$ and $15.0$. The single-answer result also matches
published GRPO on MATH500, $78.5$ against $78.0$ (Table~\ref{tab:main}). The
published numbers come from a different evaluation protocol. Under ours, on the same
$250$ problems, the sixteen-candidate sampler scores $79.2$ and \textsc{oppd} $78.4$
(Table~\ref{tab:amort}). On $250$ MATH500 problems, sixteen candidates
add $12.4$ points for the untrained model, of which \textsc{oppd} obtains $11.6$
($94\%$) with a single generation. Running the sampler on the trained model adds only
$1.6$ more (Appendix~\ref{sec:amort}).

\subsection{Comparison with training on verified rewards}
\label{sec:rlmain}

GRPO shows how large a gain is usually available when reference answers can be used.
We train the two methods from the same checkpoint to put \textsc{oppd} on that scale,
not to argue that one should replace the other, and they compose.

We train \textsc{oppd} and GRPO \citep{shao2024deepseekmath} from the same
\texttt{Qwen3-8B-Base} checkpoint for $300$ optimizer steps each. At each step, both
draw two prompts and sample sixteen answers of at most $1024$ tokens per prompt, so
each generates $9{,}600$ answers. Our GRPO is the standard version, with one
optimization pass over each batch of samples and a binary reward from the evaluation
grader. Both runs are small next to published GRPO training, $600$ prompt draws of at
most $1024$ tokens each, so the comparison bounds neither method at a larger budget. \textsc{oppd} scores higher on all three benchmarks
(Table~\ref{tab:head2head}), by $3.75$ points on MATH500, $4.00$ on GSM8K and $5.42$
on AIME, each between one and two standard errors and all three with the same sign. GRPO learns only from groups of answers whose rewards differ, and $234$ of the $600$
groups ($39\%$) gave every answer the same reward and no gradient, while every prompt
drawn produced a gradient for \textsc{oppd} (Appendix~\ref{sec:head2head}).

The two methods also combine. Applied to a \texttt{Qwen2.5-Math-7B} checkpoint that
GRPO had already trained from $50.25$ to $64.00$ on MATH500 and $59.25$ to $71.00$ on
GSM8K, \textsc{oppd} adds $9.25$ and $8.38$ more, $3.1$ and $2.8$ standard errors
(Table~\ref{tab:grpo}), and AIME pass@1 rises from $10.83\%$ with either method alone
to $12.92\%$ with both (Appendix~\ref{sec:aime}).

\subsection{Evidence that the student moves toward the power distribution}
\label{sec:mechmain}

\begin{figure}[t]
\centering
\includegraphics[width=0.56\textwidth]{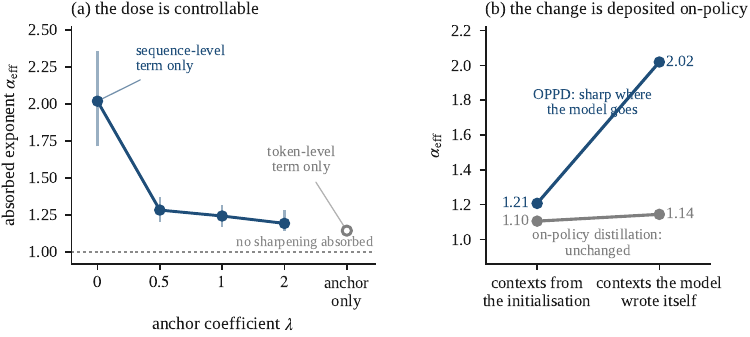}
\caption{(a) The absorbed exponent $\aeff$ decreases monotonically as the anchor
coefficient $\lambda$ increases, across five settings; bars are interquartile ranges
over per-context fits. (b) The same estimator on contexts written by the
initialization and on the model's own answers. The two agree for ordinary on-policy
distillation, while for \textsc{oppd} the estimate is much higher on its own answers,
as expected when the training signal comes only from those answers. Exact values in
Table~\ref{tab:dose}.}
\label{fig:mechanism}
\end{figure}

Training should reduce $D_\beta(\pi_\alpha \Vert \ps)$, and \eqref{eq:exchange} lets
us check this on held-out problems by measuring $\mathrm{ESS}/N$ with each model as
the proposal. The implied divergence $\hat D_2 = -\log(\mathrm{ESS}/N)$
falls by $21.6\%$ on \texttt{Qwen2.5-Math-7B} and by $26.6\%$, $24.7\%$ and $12.6\%$
at $N=4$, $16$ and $64$ on \texttt{Qwen3-8B-Base}. Along the answers of the model trained at $\lambda=0$, the
untrained model's token entropy is $0.32$ times its entropy along its own answers,
against $0.47$ for ordinary on-policy distillation (Appendix~\ref{sec:mech}).

On the $\lambda$ sweep, where \texttt{Qwen2.5-Math-1.5B} is trained with
\texttt{Qwen2.5-Math-7B} as the teacher, the absorbed exponent decreases as $\lambda$
increases across all five settings. It is $2.02$ with the sequence-level term alone
and $1.14$ for ordinary on-policy distillation, and moving
$\lambda$ from $0.5$ to $2$ changes it only from $1.28$ to $1.19$
(Figure~\ref{fig:mechanism}, Table~\ref{tab:dose}). On contexts written by the initialization
rather than by the model itself, ordinary distillation gives almost the same value and
\textsc{oppd} does not ($1.207$ against $2.019$), as expected when the signal comes
only from the student's own answers (Appendix~\ref{sec:dose}).

The best value of $\lambda$ depends on the model, and the MATH500 gain an untrained
model gets from lowering its temperature to $1/\alpha$ predicts it. This gain is
$19.9$ points on \texttt{Qwen2.5-Math-7B}, which does best at $\lambda=0$, and $0.6$
on \texttt{Llama-3.1-8B-Instruct}, which does best at $\lambda=1$ and loses $4.75$
points of MATH500 without the anchor. A large gain means that sharpening is still
available, so \texttt{Qwen3-8B-Base} ($13.0$) is trained at $\lambda=0$ (Appendix~\ref{sec:predict}).

\subsection{Other models and scale}
\label{sec:replmain}

On \texttt{Qwen3-8B-Base} at temperature $1$, \textsc{oppd} raises MATH500 from
$65.50$ to $78.25$ and GSM8K from $78.00$ to $95.13$, which are $4.3$ and $5.8$
standard errors. Both scores are above ordinary on-policy distillation ($75.00$ and
$93.00$), and best-of-4 rises on both. On \texttt{Llama-3.1-8B-Instruct} at $\lambda=1$, \textsc{oppd}
gains $3.25$ and $3.13$ points over the untrained model and $2.75$ on MATH500 over
ordinary on-policy distillation (Table~\ref{tab:repl}). On
\texttt{Qwen3-14B-Base} at $\lambda=0$ and temperature $1$, \textsc{oppd} reaches
$84.00$ on MATH500, $95.88$ on GSM8K and $16.25$ on AIME, above the untrained model's
best at either temperature ($82.00$, $92.38$ and $12.92$), and adds $18.50$ points on
MATH500 against $12.75$ on \texttt{Qwen3-8B-Base} (Table~\ref{tab:scale}).

\texttt{deepseek-math-7b-rl} \citep{shao2024deepseekmath}, from a third family, was
already trained with verified rewards, and lowering its temperature to $1/\alpha$
gains $+0.00$ points on MATH500. \textsc{oppd} at the prescribed $\lambda=1$ raises
MATH500 from $47.12$ to $51.50$ ($95\%$ interval $[+0.88, +7.88]$) and GSM8K from
$85.38$ to $87.75$ (Appendix~\ref{sec:deepseek}).

Because the method uses no reference answers, its training questions can come from
any source: training \texttt{Llama-3.1-8B-Instruct} on $4000$
\texttt{OpenR1-Math-220k} problem statements gives $51.13$ on MATH500 and $85.88$ on
GSM8K, against $52.50$ and $85.38$ with GSM8K prompts (Appendix~\ref{sec:openr1}).

\subsection{Code generation}
\label{sec:humaneval}

HumanEval \citep{chen2021codex} scores a model's implementation of a Python function,
given its signature, by running unit tests. We evaluate the \textsc{oppd} checkpoints of four
models (two for \texttt{Qwen2.5-Math-7B}) and the GRPO checkpoint of
Section~\ref{sec:rlmain}, each with its default generation configuration.

\begin{table}[t]
\caption{HumanEval, $164$ problems, four samples each, pass@1. Every trained
model was trained on mathematics questions only. $\Delta$ is trained minus untrained pass@1 on the
same problems, with a $95\%$ interval from a paired bootstrap ($100{,}000$
resamples).}
\label{tab:humaneval}
\centering\small
\begin{tabular}{llccc}
\toprule
model & training & untrained & trained & $\Delta$ [$95\%$ interval] \\
\midrule
\texttt{Qwen3-8B-Base} & \textsc{oppd}, $\lambda=0$ & $76.52$ & $\mathbf{81.86}$ & $+5.34$ [$+1.83$, $+8.84$] \\
\texttt{Qwen3-8B-Base} & GRPO & $76.52$ & $78.51$ & $+1.98$ [$-1.68$, $+5.64$] \\
\texttt{Qwen3-14B-Base} & \textsc{oppd}, $\lambda=0$ & $80.18$ & $\mathbf{83.69}$ & $+3.51$ [$+0.30$, $+6.71$] \\
\texttt{Llama-3.1-8B-Instruct} & \textsc{oppd}, $\lambda=1$ & $63.41$ & $\mathbf{68.14}$ & $+4.73$ [$+0.76$, $+8.69$] \\
\texttt{Qwen2.5-Math-7B} & \textsc{oppd}, $\lambda=0$ & $42.53$ & $46.49$ & $+3.96$ [$-0.30$, $+8.23$] \\
\texttt{Qwen2.5-Math-7B} & \textsc{oppd}, $\lambda=1$ & $42.53$ & $\mathbf{47.87}$ & $+5.34$ [$+1.07$, $+9.76$] \\
\bottomrule
\end{tabular}
\end{table}

\textsc{oppd} raises HumanEval pass@1 on all four models, by $3.51$ to $5.34$ points
(Table~\ref{tab:humaneval}). The $95\%$ interval excludes zero in every case except
\texttt{Qwen2.5-Math-7B} at $\lambda=0$, the setting of Table~\ref{tab:main}. The
GRPO checkpoint gains $1.98$ points, with an interval that includes zero.

\section{Related work}

\paragraph{Power sampling and sequential Monte Carlo.} \citet{karan2026reasoning}
introduce Metropolis--Hastings sampling from the power distribution \eqref{eq:power}
and show that it matches reinforcement-learning post-training on reasoning
benchmarks. \citet{powersmc2026} replace the serial Markov chain with the
batch-parallel sampler we use, and \citet{ji2026scalable} approximate the
next-token distribution of $\pi_\alpha$ with sampled continuations. All three
methods run at inference time. \citet{tomihari2026power} show that the power
distribution is the closed-form optimum of KL-regularized reinforcement learning when
the reward is the model's own sequence log-probability, and derive power
self-distillation, which trains a model on samples a power sampler produced and
reaches accuracy at or above power sampling at much lower inference cost. Moving the
cost of power sampling into training is therefore established. \textsc{oppd} differs
in where the training answers come from and in what is optimized: the student is the
sampler's proposal distribution, so the answers are its own and one teacher pass per
block serves both the sampler and the update, and the objective is the R\'enyi family
\eqref{eq:renyidiv} rather than supervised learning on stored samples. \textsc{oppd} trains
on-policy, the setting for which on-policy distillation was introduced
\citep{agarwal2024gkd,gu2024minillm}. SMC has also been used to steer language models toward constraints \citep{lew2023sequential} and for
inference with learned twist functions \citep{zhao2024twisted}, and
\citet{hu2024amortizing} learn to sample intractable posteriors with GFlowNet
objectives. We do not learn a twist function (Section~\ref{sec:smc}).

\paragraph{Distillation and post-training.} Sequence-level knowledge distillation
trains a student on answers generated by the teacher
\citep{kim2016sequence,hinton2015distilling}, and on-policy variants train on
answers the student generates and the teacher scores. Methods that use
verified rewards, such as GRPO \citep{shao2024deepseekmath}, need to know whether
each answer is correct. Methods that reward a model's own confidence instead include
entropy minimization \citep{agarwal2025entropy}, self-certainty
\citep{zhao2025intuitor}, negative token entropy \citep{prabhudesai2025rent} and
majority-vote rewards \citep{zuo2025ttrl}, and \citet{huang2025sharpening} analyze
self-improvement as sharpening against the model's own judgments. With the proposal
of Theorem~\ref{thm:opt} and $\beta=1$, \eqref{eq:renyi} makes our training weight a
softmax over the negative R\'enyi entropy accumulated along each answer, a confidence
weighting whose scores come from a frozen teacher and are used inside a sampler. \textsc{oppd} needs only the questions, so it applies where
those methods cannot be run, and it can also be applied to a model those methods
have already trained. \citet{yue2025rl} argue that such methods mostly re-weight
answers the model could already produce, as power sampling does.

\section{Conclusion}

\textsc{oppd} trains a model toward the sequence-level power distribution of a
teacher, so that one generation replaces the sampler. The trained model scores above
published power sampling with $64$ candidates, matches GRPO trained from the same
checkpoint and budget without reading reference answers, and improves a model that
verified rewards have already trained. One loss coefficient controls the sharpening.

\bibliography{refs}
\bibliographystyle{iclr2027_conference}

\appendix

\section{The two-token example}
\label{app:toy}

Section~\ref{sec:whytemp} states that sharpening complete answers differs from
sharpening tokens. The following example shows the difference on a small case.

Consider a vocabulary of two tokens and answers of length two, with $p(a)=0.6$,
$p(a\mid a)=0.9$ and $p(a \mid b)=p(b\mid b)=0.5$. Multiplying gives
$p(aa)=0.54$, $p(ab)=0.06$ and $p(ba)=p(bb)=0.20$. Setting $\alpha=2$:

\begin{center}\small
\begin{tabular}{lccccc}
\toprule
 & $aa$ & $ab$ & $ba$ & $bb$ & total on answers starting with $b$ \\
\midrule
power distribution, $\alpha=2$ & $0.777$ & $0.010$ & $0.107$ & $0.107$ & $0.213$ \\
temperature $\tau=1/2$ & $0.684$ & $0.008$ & $0.154$ & $0.154$ & $0.308$ \\
\bottomrule
\end{tabular}
\end{center}

The power distribution assigns less probability to answers beginning with $b$,
$0.213$ against $0.308$. It does so because every continuation of $b$ is
uncertain, so no answer beginning with $b$ can reach a high total probability.
Temperature cannot represent this, because at the first token it compares $0.6$
against $0.4$ with no information about the answers those two tokens lead to. The power
distribution accounts for how probable the continuations of each prefix are, and
temperature does not. This is the information that \textsc{oppd} transfers into the
model parameters.

\section{Extended experimental results}
\label{app:ext}

This appendix gives in full the results summarized in Section~\ref{sec:exp}.

\subsection{Running the sampler on the trained model}
\label{sec:amort}

\begin{figure}[t]
\centering
\includegraphics[width=0.55\textwidth]{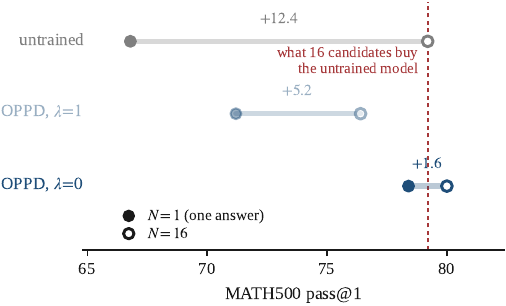}
\caption{Accuracy against inference cost on MATH500, with $n{=}250$ problems.
Sixteen candidates add $12.4$ points for the untrained model. \textsc{oppd} obtains
$11.6$ of these points with a single generation, and running the sampler on the
trained model adds only $1.6$ more. Both markers on each row are measured values, and
no interpolation between them is implied. Exact values in Table~\ref{tab:amort}.}
\label{fig:amortise}
\end{figure}

On the untrained model, sixteen candidates add $12.4$ points. With a single
generation, \textsc{oppd} obtains $11.6$ of these $12.4$ points ($94\%$), and
running the sampler on the trained model adds only $1.6$ more
(Figure~\ref{fig:amortise}). Figure~\ref{fig:method}c shows the same comparison,
between a sampler run on every query before training and one ordinary generation
after it. The gain from candidates falls from $12.4$ to $5.2$ to $1.6$ points as
training places more weight on the sequence-level term (untrained, $\lambda=1$,
$\lambda=0$). The $N{=}1$ value of $78.4$ here and the $78.50$ of
Table~\ref{tab:main} are both the MATH500 accuracy of the same trained model
generating one answer at temperature $1/\alpha$, measured under two protocols: $250$
problems with one sample, and $200$ problems with four. The two values agree to
within $0.1$ points, which checks the two protocols against each other.

\subsection{Training seeds}
\label{sec:seeds}

\begin{table}[t]
\caption{\textsc{oppd}, $\lambda=0$, self-distilled on \texttt{Qwen2.5-Math-7B} with
three training seeds, $300$ steps each. Seed $0$ is the model of
Table~\ref{tab:main}.}
\label{tab:seeds}
\centering\small
\begin{tabular}{lcccc}
\toprule
 & \multicolumn{2}{c}{MATH500} & \multicolumn{2}{c}{GSM8K} \\
\cmidrule(lr){2-3}\cmidrule(lr){4-5}
 & pass@1 & best-of-4 & pass@1 & best-of-4 \\
\midrule
\multicolumn{5}{l}{\emph{sampling at temperature $1$}}\\
seed $0$ & $73.25$ & $88.5$ & $86.50$ & $95.5$ \\
seed $1$ & $71.50$ & $88.5$ & $85.88$ & $95.5$ \\
seed $2$ & $74.63$ & $89.5$ & $89.50$ & $95.5$ \\
mean & $73.13$ & $88.8$ & $87.29$ & $95.5$ \\
\midrule
\multicolumn{5}{l}{\emph{sampling at temperature $1/\alpha$}}\\
seed $0$ & $78.50$ & $86.0$ & $91.13$ & $95.5$ \\
seed $1$ & $79.50$ & $86.0$ & $92.63$ & $94.5$ \\
seed $2$ & $81.25$ & $88.0$ & $91.88$ & $95.0$ \\
mean & $\mathbf{79.75}$ & $86.7$ & $\mathbf{91.88}$ & $95.0$ \\
\bottomrule
\end{tabular}
\end{table}

We trained the configuration of Table~\ref{tab:main} with two further random seeds
under the settings of Section~\ref{sec:exp}. Sampled at temperature $1/\alpha$, the
three runs score $78.50$ to $81.25$
on MATH500 and $91.13$ to $92.63$ on GSM8K (Table~\ref{tab:seeds}). Every run exceeds
published power sampling with $64$ candidates, by $2.4$ to $5.2$ points on MATH500
and $3.5$ to $5.0$ on GSM8K, and every run matches or exceeds published GRPO on
MATH500. At temperature $1$ the gain over the untrained model is $21.3$ to $24.4$ points on MATH500 and $26.6$ to
$30.3$ on GSM8K.

\subsection{Divergence and entropy measurements}
\label{sec:mech}

\begin{table}[t]
\caption{Left: $\mathrm{ESS}/N$ with each model as the proposal distribution, and
the divergence $\hat D_2$ to the power distribution that it implies through
\eqref{eq:exchange}. Right: the untrained model's token entropy along the answers
each trained model writes (own), along its own answers (base), and the ratio of the
two. The baseline is ordinary on-policy distillation.}
\label{tab:mech}
\centering\small
\begin{tabular}{lcc@{\hskip 2em}lcccc}
\toprule
proposal distribution & $\mathrm{ESS}/N$ & $\hat D_2$ & & own & base & ratio & MATH@1 \\
\midrule
untrained & $0.561$ & $0.579$ & baseline & $0.138$ & $0.292$ & $0.47$ & $62.25$ \\
\textsc{oppd}, $\lambda=1$ & $0.602$ & $0.508$ & \textsc{oppd}, $\lambda=1$ & $0.127$ & $0.273$ & $0.47$ & $64.00$ \\
\textsc{oppd}, $\lambda=0$ & $\mathbf{0.635}$ & $\mathbf{0.454}$ & \textsc{oppd}, $\lambda=0$ & $0.096$ & $0.296$ & $\mathbf{0.32}$ & $\mathbf{73.25}$ \\
\bottomrule
\end{tabular}
\end{table}

Training should reduce $D_\beta(\pi_\alpha \Vert \ps)$, and identity
\eqref{eq:exchange} lets us check this independently on held-out problems by
reading the ESS when each model is the proposal distribution. We write
$\hat D_2 = -\log(\mathrm{ESS}/N)$ for the divergence that this identity implies,
which is \eqref{eq:renyidiv} at $\beta = 2$ between $\pi_\alpha$ and the model used
as the proposal. The measured divergence falls by $21.6\%$, and it falls more at
$\lambda=0$ than at $\lambda=1$ (Table~\ref{tab:mech}, left).

The first row of this table needs explanation. Its proposal distribution is the
teacher at temperature $1/\alpha$, which Theorem~\ref{thm:opt} identifies as the
best proposal among those that depend only on the prefix, and yet both trained models
do better. This does not contradict the theorem, which fixes a prefix and minimizes
the variance of the next incremental importance weight. The variation that remains,
described by \eqref{eq:renyi}, comes from differences between the prefixes that
different candidates reach. A trained model cannot improve on $q^\star$ at a fixed
prefix, but it can write answers whose prefixes give more uniform importance weights.
The right half of Table~\ref{tab:mech} measures this effect. Sharpening complete
answers favors answers of high total probability, which by \eqref{eq:renyi} are
answers along which the teacher is confident throughout. Along the answers written by
the model trained at $\lambda=0$, the untrained model's token entropy is $0.32$ times
its entropy along its own answers, against $0.47$ for ordinary on-policy
distillation. Of the three models, the most accurate has the lowest ratio.

\subsection{The coefficient $\lambda$ controls the absorbed exponent}
\label{sec:dose}

The absorbed exponent decreases as $\lambda$ increases, across all five settings
(Figure~\ref{fig:mechanism}a). With the sequence-level term alone it rises by $1.02$
above one, to about half of $\alpha$, against a rise of $0.14$ for ordinary
on-policy distillation. The interquartile ranges of the two extreme settings do not
overlap. The relationship is not linear: removing the anchor entirely moves $\aeff$
from $1.28$ to $2.02$, whereas moving $\lambda$ from $0.5$ to $2$ moves it only from
$1.28$ to $1.19$. Even a small coefficient on the mass-covering anchor suppresses most
of the available sharpening, which is consistent with its forward-KL form.

\paragraph{The change appears on the answers the student writes.} The estimator
can be run on two sets of contexts. One set comes from answers written by the
initialization, and the other from answers written by the trained model itself. The
two estimates agree for ordinary on-policy distillation ($1.105$ and $1.144$) and
differ for \textsc{oppd} ($1.207$ and $2.019$) (Table~\ref{tab:dose}, right, and
Figure~\ref{fig:mechanism}b). Measuring on a model's own answers therefore does not
raise the estimate in general, only for the model trained with the sequence-level
term. On-policy training predicts this asymmetry. The training signal comes only from
answers that the student generates, so the parameters should change mainly along
those answers, and a method trained on text written by the teacher would not produce
this pattern.

\begin{figure}[t]
\centering
\includegraphics[width=0.8\textwidth]{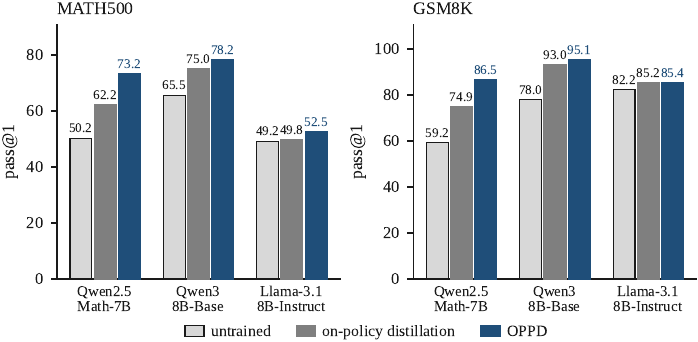}
\caption{Replication across three models, each sampled with its default generation
configuration: temperature $1$, and for \texttt{Llama-3.1-8B-Instruct} temperature
$0.6$ with nucleus (top-$p$) sampling at $p=0.9$. \textsc{oppd}
improves on both the untrained model and ordinary on-policy distillation on every
model and both benchmarks. Exact values in Table~\ref{tab:repl}; best-of-4 and
the temperature-$1/\alpha$ rows are in Appendix~\ref{app:fullacc}.}
\label{fig:replication}
\end{figure}

\subsection{Replication on two further models}
\label{sec:replication}

On \texttt{Qwen3-8B-Base}, \textsc{oppd} gains $12.75$ and $17.13$ points over the
untrained model, which are $4.3$ and $5.8$ standard errors, and best-of-4 rises on
both benchmarks. The divergence measurement gives the same result for every number of
candidates we tested: $\hat D_2$ falls from $0.494$ to $0.362$ at $N{=}4$, from
$0.712$ to $0.536$ at $N{=}16$, and from $0.851$ to $0.744$ at $N{=}64$. These are
reductions of $26.6\%$, $24.7\%$ and $12.6\%$, consistent with the $21.6\%$
measured on \texttt{Qwen2.5-Math-7B}. Values of $\hat D_2$ are not comparable
across numbers of candidates, because more candidates explore more of the
distribution, so each comparison is made at a fixed $N$.

On \texttt{Llama-3.1-8B-Instruct}, \textsc{oppd} gains $3.25$ and $3.13$ points
over the untrained model, and $2.75$ points on MATH500 over ordinary on-policy
distillation. The two methods differ more in best-of-4 than in pass@1: ordinary
on-policy distillation lowers MATH500 best-of-4 by $4.0$ points relative to the
untrained model, while \textsc{oppd} lowers it by $1.0$.

\subsection{Scaling to a larger model}
\label{sec:scale}

The models above have $7$ to $8$B parameters. To test whether the method still
works at a larger scale, we ran it on \texttt{Qwen3-14B-Base}, whose authors report
it as comparable to \texttt{Qwen2.5-32B-Base} \citep{yang2025qwen3}. All settings
were those of Section~\ref{sec:exp} except the learning rate, which was
$3\times10^{-6}$. We used $\lambda=0$, as the rule of Appendix~\ref{sec:predict}
prescribes, and evaluated both the untrained and the trained model at both decoding
temperatures.

\begin{table}[t]
\caption{\texttt{Qwen3-14B-Base}, $300$ steps at $\lambda=0$, evaluated at both
decoding temperatures. Sampled at temperature $1$, \textsc{oppd} scores above the
untrained model at either temperature in pass@1 on all three benchmarks, and in
best-of-4 on MATH500 and AIME.}
\label{tab:scale}
\centering\small
\begin{tabular}{lcccccc}
\toprule
 & \multicolumn{2}{c}{MATH500} & \multicolumn{2}{c}{GSM8K} & \multicolumn{2}{c}{AIME} \\
\cmidrule(lr){2-3}\cmidrule(lr){4-5}\cmidrule(lr){6-7}
 & pass@1 & best-of-4 & pass@1 & best-of-4 & pass@1 & best-of-4 \\
\midrule
\multicolumn{7}{l}{\emph{sampling at temperature $1$}}\\
untrained & $65.50$ & $86.5$ & $87.62$ & $97.5$ & $8.75$ & $25.0$ \\
\textsc{oppd} & $84.00$ & $\mathbf{93.0}$ & $95.88$ & $97.0$ & $\mathbf{16.25}$ & $\mathbf{28.3}$ \\
\midrule
\multicolumn{7}{l}{\emph{sampling at temperature $1/\alpha$}}\\
untrained & $82.00$ & $89.0$ & $92.38$ & $96.0$ & $12.92$ & $20.0$ \\
\textsc{oppd} & $\mathbf{85.50}$ & $91.0$ & $\mathbf{95.75}$ & $96.5$ & $15.42$ & $25.0$ \\
\bottomrule
\end{tabular}
\end{table}

With both models sampled at temperature $1/\alpha$, \textsc{oppd} gains $3.50$
points on MATH500, $3.37$ on GSM8K and $2.50$ on AIME (Table~\ref{tab:scale}). The
more informative comparison is across the two temperatures. Lowering the temperature
requires no training and raises the untrained model to $82.00$, $92.38$ and $12.92$,
its better result of the two temperatures we evaluate. \textsc{oppd} sampled at
temperature $1$, without any change to decoding, reaches $84.00$, $95.88$ and
$16.25$, above the untrained model's best on every benchmark.

\textsc{oppd} also keeps more of the model's diversity, as measured by best-of-4,
than lowering the temperature does. Lowering the temperature reduces the untrained
model's best-of-4 on two of the three benchmarks, from $97.5$ to $96.0$ on GSM8K and
from $25.0$ to $20.0$ on AIME, a cost of the higher pass@1 that a lower temperature
gives. \textsc{oppd} at temperature $1$ scores $93.0$, $97.0$ and $28.3$ in
best-of-4, above the untrained model at either temperature on MATH500 and AIME. The
gains are largest on MATH500, where \textsc{oppd} adds $18.50$ points at temperature
$1$, against the $12.75$ it adds on \texttt{Qwen3-8B-Base}.

\subsection{A third model family}
\label{sec:deepseek}

\begin{table}[t]
\caption{\texttt{deepseek-math-7b-rl}, $300$ steps at $\lambda=1$, $200$ problems and
four samples per problem. $\Delta$ is the pass@1 of \textsc{oppd} minus that of the
untrained model on the same problems, with a $95\%$ interval from a paired bootstrap over problems
($100{,}000$ resamples).}
\label{tab:deepseek}
\centering\small
\begin{tabular}{llccc}
\toprule
sampling & task & untrained & \textsc{oppd} & $\Delta$ [$95\%$ interval] \\
\midrule
\multirow{2}{*}{temperature $1$}
  & MATH500 & $47.12$ & $\mathbf{51.50}$ & $+4.38$ [$+0.88$, $+7.88$] \\
  & GSM8K   & $85.38$ & $\mathbf{87.75}$ & $+2.38$ [$+0.00$, $+4.75$] \\
\midrule
\multirow{2}{*}{temperature $1/\alpha$}
  & MATH500 & $47.12$ & $49.88$ & $+2.75$ [$-0.38$, $+6.00$] \\
  & GSM8K   & $86.00$ & $86.88$ & $+0.88$ [$-0.75$, $+2.38$] \\
\bottomrule
\end{tabular}
\end{table}

\texttt{deepseek-math-7b-rl} differs from the models above in two ways that matter
for this measurement. It comes from a third model family, and it has already been
trained with verified rewards on mathematics, so reinforcement learning has already
moved its probability toward correct answers.

Before training, lowering its sampling temperature from $1$ to $1/\alpha$ leaves
MATH500 at $47.12$ and moves GSM8K from $85.38$ to $86.00$, gains of $+0.00$ and
$+0.62$ points. The rule of Appendix~\ref{sec:predict} prescribes $\lambda=1$ for a
model that gains this little from a lower temperature, and we fixed $\lambda=1$ from
this measurement before training began.

\textsc{oppd} scores above the untrained model in every row of
Table~\ref{tab:deepseek}. At temperature $1$, MATH500 rises $4.38$ points with a
$95\%$ interval of $[+0.88, +7.88]$, and GSM8K rises $2.38$ points.

This model allows a direct comparison of the two operations of
Section~\ref{sec:whytemp}. Lowering the temperature, the token-level operation,
gives $+0.00$ points on MATH500. Sequence-level sharpening through \textsc{oppd}
gives $4.38$ points on the same benchmark and the same problems. On a model where
token-level sharpening gives no further gain, sequence-level sharpening still raises
accuracy.

\subsection{Choosing $\lambda$ from the untrained model}
\label{sec:predict}

Removing the anchor term costs \texttt{Llama-3.1-8B-Instruct} $4.75$ points of
MATH500 accuracy and $5.5$ points of best-of-4. On \texttt{Qwen2.5-Math-7B} the same
change gains $9.25$ points, so the best coefficient differs between models. A quantity
measured on the untrained model predicts which case applies: how much the model gains
from lowering its sampling temperature, which reflects how much sharpening is still
available.

\begin{center}\small
\begin{tabular}{lcc}
\toprule
model & MATH500 gain from temperature $1/\alpha$ & better $\lambda$ \\
\midrule
\texttt{Qwen2.5-Math-7B} & $+19.9$ & $0$ \\
\texttt{Qwen3-8B-Base} & $+13.0$ & $0$ (only $\lambda$ trained) \\
\texttt{Llama-3.1-8B-Instruct} & $+0.6$ & $1$ \\
\bottomrule
\end{tabular}
\end{center}

A model that gains a lot from a lower temperature still has sharpening available
and tolerates $\lambda=0$, the strongest setting. A model that gains almost nothing is
already sharp, and sharpening it further costs accuracy and best-of-4. On these three
models, a single evaluation of the untrained model identifies the better
$\lambda$.

\subsection{Applying \textsc{oppd} after training with verified rewards}
\label{sec:grpo}

Table~\ref{tab:main} compares our result with published GRPO on the same model.
Reinforcement learning with verified rewards and \textsc{oppd} might move
probability toward the same answers, in which case applying one after the other
would give no further gain. We tested this by taking a checkpoint of
\texttt{Qwen2.5-Math-7B} that had already been trained with GRPO on verified rewards
and training it further with \textsc{oppd}. All other settings were those of
Section~\ref{sec:exp}.

\begin{table}[t]
\caption{\textsc{oppd} applied to a checkpoint that has already been trained with
verified rewards. The second row is that checkpoint. GRPO moves the untrained model
from $50.25$ to $64.00$ on MATH500, and \textsc{oppd} adds a further $9.25$
points.}
\label{tab:grpo}
\centering\small
\begin{tabular}{lcc}
\toprule
 & MATH500 pass@1 & GSM8K pass@1 \\
\midrule
untrained model & $50.25$ & $59.25$ \\
after GRPO with verified rewards & $64.00$ & $71.00$ \\
\midrule
GRPO, then \textsc{oppd}, $\lambda=1$ & $72.13$ & $79.75$ \\
GRPO, then \textsc{oppd}, $\lambda=0$ & $\mathbf{73.25}$ & $\mathbf{79.38}$ \\
\bottomrule
\end{tabular}
\end{table}

Training with verified rewards moves the untrained model from $50.25$ to $64.00$ on
MATH500 and from $59.25$ to $71.00$ on GSM8K. \textsc{oppd} then adds a further
$9.25$ and $8.38$ points, which are $3.1$ and $2.8$ standard errors
(Table~\ref{tab:grpo}). Both settings of $\lambda$ give gains of a similar size,
$8.12$ and $8.75$ points at $\lambda=1$, so this model is not sensitive to the
anchor coefficient in the way that the instruction-tuned model of
Appendix~\ref{sec:predict} was.

The two procedures can therefore be combined: training with verified rewards does
not already provide what \textsc{oppd} adds, although both start from the same
untrained model and improve the same benchmarks. Because the two gains add up,
\textsc{oppd} is useful both where verified rewards are unavailable, as in
Table~\ref{tab:main}, and as a further step after training with them.

\begin{figure}[t]
\centering
\includegraphics[width=0.7\textwidth]{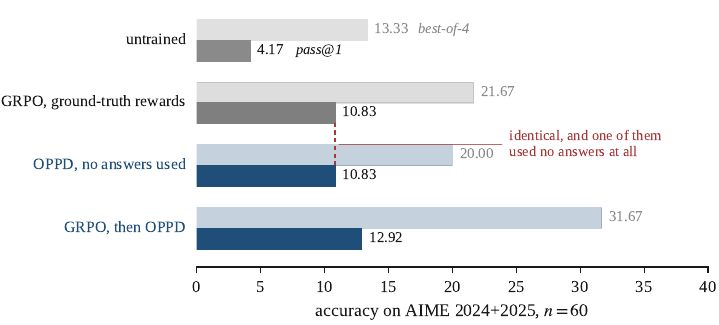}
\caption{AIME 2024 and 2025 pooled, $n{=}60$ problems, four answers each.
\textsc{oppd} reaches the same pass@1 as training with verified rewards without
using reference answers, and applying it after that training exceeds either method
alone. Best-of-4 rises to $31.67$, so the combined model solves problems that the
untrained model does not solve in four attempts. The same evaluation on
\texttt{Qwen3-8B-Base} gives $7.08$ untrained and $14.17$ with \textsc{oppd}. Exact
values in Table~\ref{tab:aime}.}
\label{fig:aime}
\end{figure}

\subsection{Equal-budget comparison with GRPO}
\label{sec:head2head}

\begin{figure}[t]
\centering
\includegraphics[width=0.50\textwidth]{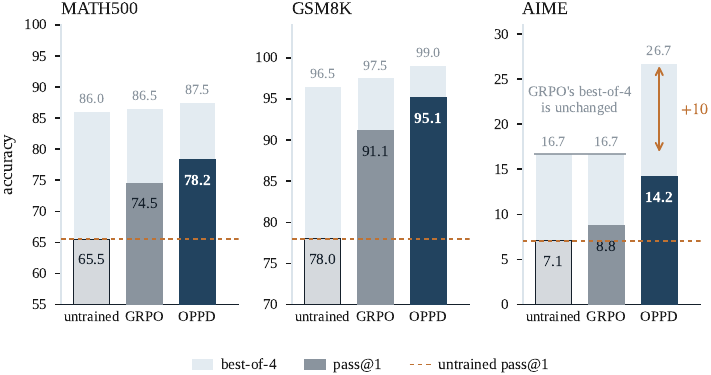}
\caption{\textsc{oppd} and GRPO trained from the same \texttt{Qwen3-8B-Base}
checkpoint with an equal budget. Solid bar: pass@1; pale extension: best-of-4;
dashed line: untrained pass@1. 
}
\label{fig:head2head}
\end{figure}

Appendix~\ref{sec:grpo} applied \textsc{oppd} to a checkpoint that others had
trained with GRPO. This shows that the two can be combined, but it does not say how
large the gain from \textsc{oppd} is beside a method that reads reference answers.
Here we train both ourselves, from the same starting model and with the same budget.
The comparison uses one standard implementation of GRPO at one budget, and we draw
no conclusion from it about the wider family of reinforcement-learning methods that
use verified rewards.

\paragraph{What is held equal.} Both methods run for $300$ optimizer steps. At
each step both draw two prompts and sample sixteen answers of at most $1024$ tokens
per prompt, so each method generates $9{,}600$ answers over the run. GRPO
additionally uses a reference answer for every training question, which
\textsc{oppd} never does.

Our GRPO implementation is the standard one, with a single optimization pass over
each batch of samples. For each prompt it samples sixteen answers, called a group,
and scores each with a binary reward from the grader used for evaluation. The
advantage of each answer is its reward minus the group mean, divided by the group
standard deviation. With a single pass, the model that generated the samples is the
model being updated, so the ratio between the new and the sampling probabilities is
one and the clipping range is inactive.

\textsc{oppd} scores higher on all three benchmarks
(Figure~\ref{fig:head2head}), by $3.75$ points on MATH500, $4.00$ on GSM8K and
$5.42$ on AIME. Each margin is between one and two standard errors, so none is
conclusive on its own, but all three favor the method that uses no reference
answers.

Best-of-4, shown as the pale extensions in Figure~\ref{fig:head2head},
separates the two methods more clearly than pass@1. On MATH500 and GSM8K both
methods raise it, GRPO by $0.5$ and $1.0$ points and \textsc{oppd} by $1.5$ and
$2.5$. On AIME they differ: GRPO leaves best-of-4 at $16.67$, the same as the
untrained model, while \textsc{oppd} raises it to $26.67$. A method that raises
pass@1 without raising best-of-4 selects more reliably among answers the model could
already produce. Raising both means that the model solves problems it did not solve
in four attempts before, and on the hardest of the three benchmarks only
\textsc{oppd} does so. AIME has only $60$ problems, so the pass@1 margins there are
imprecise.

\paragraph{Groups that give GRPO no training signal.} GRPO learns from a group only
when the rewards in it differ. If all sixteen answers earn the same reward, the group
mean equals every reward, all advantages are zero, and the prompt contributes no
gradient. Our logs record how often this happened: $234$ of GRPO's $600$
prompt draws ($39\%$) had uniform rewards and produced no gradient. These draws still cost computation, because the answers had
to be generated before their rewards could be compared.

This fraction is higher than the average reward rate would suggest. Late in
training, the mean reward ranged from $0.06$ to $0.97$ across steps. If every answer
succeeded independently at the average of these rates, fewer than $2\%$ of groups of
sixteen would have uniform rewards. Observing $39\%$ therefore means that difficulty
is not spread evenly across prompts. A given problem tends to be either easy for the
model, so that most of the sixteen answers succeed, or too hard, so that almost none
do. Both cases tend to give uniform rewards and zero advantages, so mainly problems
of intermediate difficulty provide a learning signal.

\textsc{oppd} does not have this problem, because its training weights come from the
teacher's probabilities rather than from a correctness check. Sixteen answers that
are all incorrect still differ in the total probability the teacher assigns them, so
their training weights still differ. The measured ESS stayed between $0.39N$ and
$0.64N$ throughout training, so no group's weight concentrated on a single candidate,
and every prompt drawn produced a gradient.

The two methods also place different requirements on the training data. GRPO needs a
reference answer that a grader can check, and of the $12{,}000$
\texttt{OpenR1-Math-220k} items we examined, $8{,}236$ had one. The rest give their
answer in a form that the grader cannot match to a single expression, such as a pair
of values or a multiple-choice letter. \textsc{oppd} uses only the question, so all
$12{,}000$ items were available to it.

\subsection{Results on AIME}
\label{sec:aime}

MATH500 and GSM8K are solved often enough by the untrained model that a large
part of each score is already available before training. We therefore also
evaluate on AIME, a competition benchmark on which a $7$B model rarely succeeds.
We pool the $2024$ and $2025$ problem sets, giving $60$ problems, and generate
four answers for each. All rows use the same $3500$-token generation limit, which
keeps the prompt and the answer inside the model's context window.

The untrained model solves $4.17\%$ of AIME problems on a single attempt. Training
it with GRPO on verified rewards raises that to $10.83\%$, and \textsc{oppd}, which
uses no reference answers, reaches the same $10.83\%$ (Figure~\ref{fig:aime}). The
two methods reach this number with different requirements: GRPO needs to know
whether every training answer was correct, while \textsc{oppd} needs only the
questions.

Applying \textsc{oppd} after GRPO exceeds either method alone, reaching $12.92\%$.
Best-of-4 separates the three more clearly, at $21.67$ for GRPO, $20.00$ for
\textsc{oppd}, and $31.67$ for the two together. The combined model solves
$31.67\%$ of these problems within four attempts, where the untrained model solves
$13.33\%$, so it solves problems that the untrained model does not solve in four
attempts. Selecting among answers the untrained model already produces cannot give
such a gain.

\paragraph{The AIME result holds on a second model.} We repeated the evaluation on
\texttt{Qwen3-8B-Base}, a stronger and more recent starting point than
\texttt{Qwen2.5-Math-7B}. The untrained model solves $7.08\%$ of AIME problems on a
single attempt, and \textsc{oppd} raises that to $14.17\%$, with best-of-4 going
from $16.67$ to $26.67$. Single-answer accuracy at least doubles on both models, and
$14.17\%$ is the highest AIME result among our $7$B and $8$B models, so the pattern
does not depend on starting from a model that rarely solves these problems.

The gain also transfers to problems harder than those used in training. The training
prompts are GSM8K questions, which are grade-school word problems, while the gain
appears on competition problems that the model solves only occasionally.

\subsection{Training on a different prompt source}
\label{sec:openr1}

Because the method uses no reference answers, the training questions can come from
any source. We repeated the training on \texttt{Llama-3.1-8B-Instruct} using $4000$
problem statements from \texttt{OpenR1-Math-220k} \citep{openr1}. A corpus
assembled from competition sources may contain problems that also appear in the
benchmark. We therefore removed every training prompt whose normalized form
matched an evaluation problem, so no benchmark question appeared during training.

Training on this corpus gives results within one standard error of training on
GSM8K questions: MATH500 is $51.13$ against $52.50$ and GSM8K is $85.88$ against
$85.38$, differences of $1.4$ and $0.5$ points against a standard error of $2.9$.
Of the configurations trained on this model, this one also has the highest
best-of-4 (tied on GSM8K): $69.0$ against the untrained model's $69.5$ on MATH500,
and $94.0$ against $92.5$ on GSM8K, while pass@1 rises on both benchmarks. One training run per model, on questions from any
source, is therefore sufficient. The method needs a supply of questions, not
questions matched to the evaluation.

\subsection{Training against the student itself or a larger teacher}
\label{sec:selfteacher}

Our main results use the student as its own teacher. Section~\ref{sec:exp} chose
self-distillation for a measurement reason: $\aeff(\pb \to \pt)$ is then exactly
$1.000$, so no measured sharpening can come from copying a stronger model. The choice
also follows from the purpose of the method:
\textsc{oppd} redistributes probability that the student already assigns, so the
distribution that should be sharpened is the student's own.

Standard distillation practice suggests the opposite, namely that a larger
teacher should do better. We tested this on
\texttt{Qwen2.5-Math-1.5B}, training it twice for $300$ steps with $\lambda=1$ and
otherwise identical settings, once against \texttt{Qwen2.5-Math-7B} as teacher and
once against itself.

\begin{table}[t]
\caption{\texttt{Qwen2.5-Math-1.5B} trained twice from the same starting point,
$300$ steps each, differing only in which model scores the candidates. Scoring
against the student's own distribution matches or exceeds scoring against a
$4\times$ larger teacher.}
\label{tab:selfteacher}
\centering\small
\begin{tabular}{lcccc}
\toprule
 & \multicolumn{2}{c}{MATH500} & \multicolumn{2}{c}{GSM8K} \\
\cmidrule(lr){2-3}\cmidrule(lr){4-5}
teacher & pass@1 & best-of-4 & pass@1 & best-of-4 \\
\midrule
none (untrained) & $49.75$ & $77.0$ & $57.88$ & $88.5$ \\
\texttt{Qwen2.5-Math-7B} & $58.25$ & $\mathbf{82.5}$ & $67.63$ & $89.5$ \\
the student itself & $\mathbf{59.88}$ & $\mathbf{82.5}$ & $\mathbf{73.25}$ & $\mathbf{91.0}$ \\
\bottomrule
\end{tabular}
\end{table}

Both settings improve on the untrained model, by $8.50$ and $9.75$ points with
the larger teacher and by $10.13$ and $15.37$ points with the student itself
(Table~\ref{tab:selfteacher}). Self-distillation is ahead by $5.62$ points on
GSM8K, which is $1.5$ standard errors, and by $1.63$ on MATH500, which is within
noise. The two benchmarks together show that scoring against the student's own
distribution loses nothing relative to a teacher four times its size, and on one
benchmark it does better.

This agrees with how the method works. \textsc{oppd} uses the scoring model's ranking
of complete answers, and that ranking indicates where to move probability only if
the probability belongs to the model being trained. A larger teacher defines a power
distribution that is not the student's own, so part of the gradient goes into
closing the gap between the two models rather than into sharpening the student.
Self-distillation removes this competing objective, so the method does not need a
larger model to learn from.

Both runs used $\lambda=1$, chosen for the larger-teacher setting. We did not test
whether $\lambda=0$ changes their order, or how the comparison changes as the gap
between teacher and student grows.

\subsection{Where the sampler's gain comes from}
\label{sec:selection}

Grading every candidate separates two things that the sampler does: it discards
candidates while answers are being written, and at the end it re-ranks the finished
answers by importance weight. If the importance weights carried information about
correctness, the total normalized weight on correct candidates would exceed the
fraction of candidates that are correct. Across eleven measurements spanning three
models and three numbers of candidates, the two agree to within $0.5$ points in
every case. This includes self-distillation, where teacher and student share an
architecture. Choosing the answer by weighted draw, by highest weight, or by majority
vote gives the same result throughout (Appendix~\ref{app:selection}).

The teacher's probabilities therefore carry no useful information for ranking
finished answers. All of the sampler's gain comes from discarding candidates while
the answers are still being written and their continuations still differ. Once the
answers are complete, the teacher's probabilities no longer separate correct from
incorrect candidates. This has a consequence for a family of methods: re-ranking
complete answers by model probability should not be expected to help, whereas
intervening during generation should.

\section{Proofs}
\label{app:proofs}

\begin{proof}[Proof of Theorem~\ref{thm:opt}]
Fix a prefix and write $p(v) := \pt(v\mid x, y_{<t})$. For a proposal $q$ the
incremental importance weight is $\omega(v) = p(v)^{\alpha}/q(v)$, and its mean is proposal-independent:
$\mathbb{E}_q[\omega] = \sum_v q(v) p(v)^{\alpha}/q(v) = Z_t$. Minimizing the
variance is therefore equivalent to minimizing
$\mathbb{E}_q[\omega^2] = \sum_v p(v)^{2\alpha}/q(v)$ subject to
$\sum_v q(v) = 1$. With Lagrangian
$\sum_v p(v)^{2\alpha}/q(v) + \mu(\sum_v q(v) - 1)$, stationarity gives
$-p(v)^{2\alpha}/q(v)^2 + \mu = 0$, hence $q(v) = p(v)^{\alpha}/\sqrt{\mu}$, and
normalization fixes $\sqrt{\mu} = Z_t$. Substituting, $\omega(v) \equiv Z_t$ for
all $v$, so the conditional variance is zero and the minimizer is unique.
\end{proof}

\begin{proof}[Proof of Theorem~\ref{thm:master}]
Let $S = \sum_y \pi_\alpha(y)^{\beta}\ps(y)^{1-\beta}$, so
$D_\beta = \tfrac{1}{\beta-1}\log S$. Differentiating and using
$\nabla_\phi \ps = \ps \nabla_\phi \log \ps$,
\[
\nabla_\phi S = \sum_y \pi_\alpha^{\beta}(1-\beta)\ps^{-\beta}\nabla_\phi \ps
= (1-\beta)\sum_y \pi_\alpha^{\beta}\ps^{1-\beta}\nabla_\phi \log\ps .
\]
Then $\nabla_\phi D_\beta = \tfrac{1}{\beta-1}\nabla_\phi S / S
= -\sum_y \tfrac{\pi_\alpha^{\beta}\ps^{1-\beta}}{S}\nabla_\phi\log\ps
= -\mathbb{E}_{\bar\pi^{(\beta)}}[\nabla_\phi \log \ps]$, since
$\tfrac{1-\beta}{\beta-1} = -1$ for every $\beta \ne 1$; the $\beta \to 1$ case
follows by continuity and gives the gradient of $\mathrm{KL}(\pi_\alpha \Vert \ps)$.
The
normalizer $Z_\alpha$ does not depend on $\phi$ and drops out.
\end{proof}

\section{Tables corresponding to the figures}
\label{app:figtables}

Figures~\ref{fig:amortise}, \ref{fig:mechanism}, \ref{fig:replication} and
\ref{fig:aime} present these measurements graphically. The exact values follow.

\begin{table}[t]
\caption{Each model used as the proposal distribution of the sampler, on MATH500
with $n{=}250$ and every candidate graded. Generating a single answer,
\textsc{oppd} at $\lambda=0$ comes within $0.8$ points of the untrained model with
sixteen candidates.}
\label{tab:amort}
\centering\small
\begin{tabular}{lccccc}
\toprule
proposal distribution & $N{=}1$ & $N{=}16$ & gain from candidates & best candidate & $\mathrm{ESS}/N$ \\
\midrule
untrained & $66.8$ & $79.2$ & $+12.4$ & $82.4$ & $0.561$ \\
\textsc{oppd}, $\lambda=1$ & $71.2$ & $76.4$ & $+5.2$ & $79.6$ & $0.602$ \\
\textsc{oppd}, $\lambda=0$ & $\mathbf{78.4}$ & $80.0$ & $\mathbf{+1.6}$ & $81.6$ & $0.635$ \\
\bottomrule
\end{tabular}
\end{table}

\begin{table}[t]
\caption{Left: the absorbed exponent decreases as $\lambda$ increases, across five
settings, with \texttt{Qwen2.5-Math-1.5B} as the student, $\alpha=4$, and contexts
taken from the answers each model writes itself. Right: the same estimator on two
sets of contexts. Only the model trained with the sequence-level term measures higher
on its own answers. The baseline is ordinary on-policy distillation.}
\label{tab:dose}
\centering\small
\begin{tabular}{lcc@{\hskip 2.5em}lccc}
\toprule
 & $\aeff(\pb\!\to\!\ps)$ & IQR & & contexts from base & own answers & ratio \\
\midrule
\textsc{oppd}, $\lambda=0$ & $\mathbf{2.019}$ & $[1.713, 2.356]$ & baseline & $1.105$ & $1.144$ & $1.04$ \\
\textsc{oppd}, $\lambda=0.5$ & $1.283$ & $[1.204, 1.372]$ & \textsc{oppd}, $\lambda=0$ & $1.207$ & $\mathbf{2.019}$ & $\mathbf{1.67}$ \\
\textsc{oppd}, $\lambda=1$ & $1.243$ & $[1.172, 1.319]$ & & & & \\
\textsc{oppd}, $\lambda=2$ & $1.193$ & $[1.144, 1.286]$ & & & & \\
baseline & $1.144$ & $[1.110, 1.178]$ & & & & \\
\bottomrule
\end{tabular}
\end{table}

\begin{table}[t]
\caption{Replication at each model's default generation configuration. Left:
\texttt{Qwen3-8B-Base}, temperature $1$. Right: \texttt{Llama-3.1-8B-Instruct},
temperature $0.6$ with nucleus sampling at $p=0.9$. On
\texttt{Llama-3.1-8B-Instruct}, lowering the sampling temperature to $1/\alpha$
gains only $0.6$ points on MATH500, so gains measured here come from training rather
than from decoding. The baseline is ordinary on-policy distillation.}
\label{tab:repl}
\centering\small
\begin{tabular}{lcccc@{\hskip 1.5em}lcccc}
\toprule
\multicolumn{5}{c}{\texttt{Qwen3-8B-Base}} & \multicolumn{5}{c}{\texttt{Llama-3.1-8B-Instruct}} \\
\cmidrule(lr){1-5}\cmidrule(lr){6-10}
 & \multicolumn{2}{c}{MATH} & \multicolumn{2}{c}{GSM} & & \multicolumn{2}{c}{MATH} & \multicolumn{2}{c}{GSM} \\
 & @1 & @4 & @1 & @4 & & @1 & @4 & @1 & @4 \\
\midrule
untrained & $65.50$ & $86.0$ & $78.00$ & $96.5$ & untrained & $49.25$ & $69.5$ & $82.25$ & $92.5$ \\
baseline & $75.00$ & $86.5$ & $93.00$ & $97.5$ & baseline & $49.75$ & $65.5$ & $85.20$ & $94.0$ \\
\textsc{oppd}, $\lambda{=}0$ & $\mathbf{78.25}$ & $\mathbf{87.5}$ & $\mathbf{95.13}$ & $\mathbf{99.0}$ & \textsc{oppd}, $\lambda{=}0$ & $47.75$ & $63.0$ & $85.75$ & $91.5$ \\
 & & & & & \textsc{oppd}, $\lambda{=}1$ & $\mathbf{52.50}$ & $68.5$ & $85.38$ & $93.0$ \\
 & & & & & \;\;\emph{OpenR1 prompts} & $51.13$ & $\mathbf{69.0}$ & $\mathbf{85.88}$ & $\mathbf{94.0}$ \\
\bottomrule
\end{tabular}
\end{table}

\begin{table}[t]
\caption{AIME 2024 and 2025 pooled, $n{=}60$ problems, four answers each.
\textsc{oppd} reaches the same pass@1 as training with verified rewards without
using reference answers. Applying it after that training exceeds either method
alone, and best-of-4 rises to $31.67$, so the combined model solves problems that the
untrained model does not solve in four attempts.}
\label{tab:aime}
\centering\small
\begin{tabular}{lcc}
\toprule
 & pass@1 & best-of-4 \\
\midrule
untrained \texttt{Qwen2.5-Math-7B} & $4.17$ & $13.33$ \\
after GRPO with verified rewards & $10.83$ & $21.67$ \\
\textsc{oppd}, no reference answers & $10.83$ & $20.00$ \\
GRPO, then \textsc{oppd}, $\lambda=0$ & $12.92$ & $\mathbf{31.67}$ \\
\midrule
untrained \texttt{Qwen3-8B-Base} & $7.08$ & $16.67$ \\
\texttt{Qwen3-8B-Base} with \textsc{oppd} & $\mathbf{14.17}$ & $26.67$ \\
\bottomrule
\end{tabular}
\end{table}

\begin{table}[t]
\caption{\textsc{oppd} and GRPO trained from the same \texttt{Qwen3-8B-Base}
checkpoint for $300$ steps each, generating $9{,}600$ answers each, then evaluated
with the same code. GRPO uses a reference answer for every training question, and
\textsc{oppd} uses none. The two runs used the same \texttt{GSM8K} prompt set.}
\label{tab:head2head}
\centering\small
\begin{tabular}{lcccccc}
\toprule
 & \multicolumn{2}{c}{MATH500} & \multicolumn{2}{c}{GSM8K} & \multicolumn{2}{c}{AIME} \\
\cmidrule(lr){2-3}\cmidrule(lr){4-5}\cmidrule(lr){6-7}
 & pass@1 & best-of-4 & pass@1 & best-of-4 & pass@1 & best-of-4 \\
\midrule
untrained & $65.50$ & $86.0$ & $78.00$ & $96.5$ & $7.08$ & $16.67$ \\
GRPO, reference answers & $74.50$ & $86.5$ & $91.13$ & $97.5$ & $8.75$ & $16.67$ \\
\textsc{oppd}, no reference answers
  & $\mathbf{78.25}$ & $\mathbf{87.5}$ & $\mathbf{95.13}$ & $\mathbf{99.0}$
  & $\mathbf{14.17}$ & $\mathbf{26.67}$ \\
\bottomrule
\end{tabular}
\end{table}

\section{Notation and definitions}
\label{app:notation}

This section defines the terms used in the paper, for readers who are not
familiar with all of them.

\paragraph{Prompts, answers, probabilities.} $x$ is the prompt, $y=(y_1,\dots,y_T)$
an answer of $T$ tokens, and $y_{<t}$ the prefix before position $t$. A token is a
short piece of text. A language model factors $p(y\mid x) = \prod_t p(y_t \mid x,
y_{<t})$, and $p(\cdot \mid x, y_{<t})$, as a function of the next token, is a
\emph{next-token distribution}.

\paragraph{Temperature.} Sampling at temperature $\tau$ replaces each next-token
distribution $p(\cdot)$ by one proportional to $p(\cdot)^{1/\tau}$. Temperature
$\tau=1$ is ordinary sampling, and $\tau<1$ makes likely tokens more likely, which
is token-level \emph{sharpening}. The power distribution sharpens complete answers
instead (Section~\ref{sec:whytemp}).

\paragraph{pass@1 and best-of-4.} We draw four independent answers per problem.
\emph{pass@1} is the fraction correct, averaged over problems; it estimates the
chance a single answer is correct. \emph{best-of-4} counts a problem solved if at
least one of the four is correct. A method that raises pass@1 while best-of-4
holds or rises is solving more problems. A method that raises pass@1 while best-of-4
falls is producing the same answer more consistently. The gap between the two
metrics shows how much the model's success depends on which sample is drawn.

\paragraph{Entropy and nats.} Entropy measures how spread out a distribution is,
and it is zero when one outcome is certain. We report it in \emph{nats}, the unit
that natural logarithms give. For the models studied here, a next-token entropy of a
few tenths of a nat means that the model is already confident.

\paragraph{Divergences.} $\mathrm{KL}(P\Vert Q)$ measures how different $P$ is
from $Q$. It is zero if and only if the two agree, and it is not symmetric, so the
two orderings behave differently as objectives. Forward KL, $\mathrm{KL}(P \Vert
Q)$ with $P$ the distribution being approximated, is \emph{mass-covering}: it
heavily penalizes $Q$ for assigning little probability where $P$ assigns some, and
it tolerates $Q$ being broad. Reverse KL, $\mathrm{KL}(Q \Vert P)$, is
\emph{mode-seeking}: it penalizes $Q$ for assigning probability where $P$ assigns
little. The R\'enyi divergence
$D_\beta(P\Vert Q) = \tfrac{1}{\beta-1}\log\sum_y P(y)^{\beta}Q(y)^{1-\beta}$
is a family indexed by $\beta>0$, equal to KL at $\beta=1$; it is zero exactly when
$P = Q$ and positive otherwise. The R\'enyi entropy
is $H_\alpha(p) = \tfrac{1}{1-\alpha}\log\sum_v p(v)^{\alpha}$.

\paragraph{Distillation.} Distillation trains a \emph{student} $\ps$, with
parameters $\phi$, to behave like a \emph{teacher} $\pt$. In \emph{off-policy}
distillation the teacher writes examples and the student reproduces them. In
\emph{on-policy} distillation the student writes answers and the teacher scores
them, so the student learns to continue from its own mistakes.
\emph{Self-distillation} initializes both from the same model, which rules out that
a measured change comes from the student copying a stronger model.

\paragraph{Stop-gradient.} $[\cdot]_{\mathrm{sg}}$ denotes a quantity held fixed
during differentiation.

\paragraph{Uncertainty.} With $200$ problems and four samples each, the paired
standard error on a difference between two methods on the same problems is
$\approx 2.9$ points, written $1\sigma$. Differences of $1\sigma$ or less are not
treated as real. For per-context estimates we report the interquartile range (IQR),
the range of the middle half of the measurements.

\section{Extended background on the sampler}
\label{app:smc}

\paragraph{Importance sampling.} To estimate averages under a distribution
$\pi = \gamma/Z$, where $\gamma$ can be evaluated but $Z$ cannot, draw samples from
a proposal distribution $q$ and weight each sample by $w(y)=\gamma(y)/q(y)$, which
measures how much more $\pi$ favors the sample than $q$ does. The self-normalized
estimate is
$\mathbb{E}_\pi[f] \approx \sum_i w(y^{(i)})f(y^{(i)}) / \sum_i w(y^{(i)})$, and
the division removes $Z$. This estimator is consistent rather than unbiased, with
bias of order $1/N$.

\paragraph{Computing the weights incrementally.} When importance sampling is
applied to complete answers in one step, the weights vary over many orders of
magnitude and one sample dominates. Defining unnormalized distributions over
prefixes, $\gamma_t(y_{1:t}) = \pt(y_{1:t}\mid x)^{\alpha}$, and updating the
weights by \eqref{eq:incw} gives the same total weight $\gamma_T/q$, while letting
the sampler monitor the weights and resample during generation.

\paragraph{Resampling.} When the ESS falls below $\kappa N$ (we use $\kappa=0.5$),
the candidates are resampled with probability proportional to their normalized
importance weights, and all weights are reset to be equal. We use systematic
resampling, which is unbiased and has lower variance than multinomial resampling
\citep{johansen2009tutorial}. Duplicated candidates diverge again because their
next tokens are sampled independently. Any unbiased resampling scheme leaves the
distribution being approximated unchanged and affects only the variance and the
diversity of the candidates.

\paragraph{What the weights select for.} Equation~\eqref{eq:renyi} states that,
under the variance-optimal proposal, a candidate's log importance weight is
$(1-\alpha)$ times the sum of the R\'enyi entropies of the teacher's next-token
distributions along its answer. Since $\alpha>1$ makes $(1-\alpha)$ negative,
candidates whose prefixes the teacher is confident about keep their weight, and
candidates that pass through uncertain positions lose it. No reward model, verifier
or reference answer is used at any point.

\section{Why a weighted token-level KL cannot reach the power distribution}
\label{app:wrongloss}

An alternative objective multiplies a per-token KL divergence to the teacher at
temperature $1/\alpha$ by the training weight of each answer,
$\sum_{i,t} [\bar A^{(i)}]_{\mathrm{sg}}\,
\mathrm{KL}(\pi^{(\alpha)}_{\mathrm{T}}(\cdot\mid y^{(i)}_{<t}) \Vert
\ps(\cdot \mid y^{(i)}_{<t}))$. This objective cannot reach $\pi_\alpha$. A
divergence multiplied by a nonnegative scalar has the same minimizer as the
divergence itself, so for every context $c$ that is visited, the optimal
$\ps(\cdot\mid c)$ is $\pi^{(\alpha)}_{\mathrm{T}}(\cdot \mid c)$ whatever the
weights are. Weights placed outside a divergence change \emph{which contexts
contribute}, but they never change \emph{the distribution matched at each context}.
The attainable optimum is therefore the teacher at temperature $1/\alpha$, which
Section~\ref{sec:whytemp} shows differs from $\pi_\alpha$.

The underlying reason is that the next-token distribution of $\pi_\alpha$ is
proportional to $\pt(v \mid y_{<t})^{\alpha}\,V_t(y_{1:t}v)$, where the twist
function $V_t = \sum_{y_{>t}}\pt(y_{>t}\mid y_{1:t})^{\alpha}$ sums over all
continuations. Setting $V_t \equiv \mathrm{const}$ leaves exactly the teacher at
temperature $1/\alpha$, and a factor missing \emph{inside} the divergence cannot be
supplied by a weight \emph{outside} it. Equation~\eqref{eq:loss} instead places the
weights on a log-likelihood, where they change the distribution being fitted.

\section{The derivation in full}
\label{app:derivation}

This appendix derives the loss from the objective step by step. It repeats material
from Section~\ref{sec:method} in a form that can be checked line by line, and it
contains no new results.

\subsection{What we are minimizing}

Fix a prompt $x$. The student $\ps$, with parameters $\phi$, assigns probability
$\ps(y \mid x)$ to each complete answer $y$. The power distribution of
\eqref{eq:power}, with the teacher $\pt$ in place of $p$, is
\[
\pi_\alpha(y \mid x) \;=\; \frac{\pt(y \mid x)^{\alpha}}{Z_\alpha(x)},
\qquad
Z_\alpha(x) \;=\; \sum_{y'} \pt(y' \mid x)^{\alpha}.
\]
We choose an order $\beta > 0$ and minimize the R\'enyi divergence
\eqref{eq:renyidiv} between the power distribution and the student,
\[
\mathcal{J}(\phi) \;=\; D_\beta\big(\pi_\alpha \,\Vert\, \ps\big)
\;=\; \frac{1}{\beta - 1}\,\log S(\phi),
\qquad
S(\phi) \;=\; \sum_y \pi_\alpha(y)^{\beta}\, \ps(y)^{1-\beta},
\]
where the dependence on $x$ is omitted. Minimizing $\mathcal{J}$ is
equivalent to minimizing $\log S$ when $\beta > 1$ and to maximizing it when
$\beta < 1$; the factor $\tfrac{1}{\beta-1}$ handles the sign, which is why the
gradient below has one form for every $\beta$.

\subsection{The gradient of that objective}
\label{app:d2}

Differentiate $S$ with respect to $\phi$. Only the factor $\ps(y)^{1-\beta}$
depends on $\phi$, and $\tfrac{\partial}{\partial \phi}\ps^{1-\beta}
= (1-\beta)\,\ps^{-\beta}\,\nabla_\phi \ps$, so
\[
\nabla_\phi S \;=\; \sum_y \pi_\alpha(y)^{\beta}\,(1-\beta)\,\ps(y)^{-\beta}\,
\nabla_\phi \ps(y).
\]
Use the identity $\nabla_\phi \ps = \ps \, \nabla_\phi \log \ps$, which holds
wherever $\ps > 0$, to replace $\nabla_\phi \ps(y)$:
\[
\nabla_\phi S \;=\; (1-\beta)\sum_y \pi_\alpha(y)^{\beta}\, \ps(y)^{1-\beta}\,
\nabla_\phi \log \ps(y).
\]
Now differentiate $\mathcal{J} = \tfrac{1}{\beta-1}\log S$:
\[
\nabla_\phi \mathcal{J}
\;=\; \frac{1}{\beta-1}\cdot\frac{\nabla_\phi S}{S}
\;=\; \frac{1-\beta}{\beta-1} \sum_y
\frac{\pi_\alpha(y)^{\beta}\,\ps(y)^{1-\beta}}{S}\, \nabla_\phi \log \ps(y).
\]
The prefactor $\tfrac{1-\beta}{\beta-1}$ equals $-1$ for every $\beta \neq 1$.
Define
\[
\bar\pi^{(\beta)}(y) \;=\; \frac{\pi_\alpha(y)^{\beta}\,\ps(y)^{1-\beta}}{S},
\]
which is nonnegative and sums to one over $y$, so it is a probability
distribution. Then
\begin{equation}
\nabla_\phi D_\beta(\pi_\alpha \Vert \ps)
\;=\; -\sum_y \bar\pi^{(\beta)}(y)\, \nabla_\phi \log \ps(y)
\;=\; -\,\mathbb{E}_{\bar\pi^{(\beta)}}\!\left[\nabla_\phi \log \ps(y)\right],
\label{eq:appgrad}
\end{equation}
which is Theorem~\ref{thm:master}. The case $\beta = 1$ follows by continuity and
gives the gradient of $\mathrm{KL}(\pi_\alpha \Vert \ps)$.

Substituting $\pi_\alpha = \pt^{\alpha}/Z_\alpha$ into the definition of
$\bar\pi^{(\beta)}$ gives
$\bar\pi^{(\beta)} \propto \pt^{\alpha\beta}\, \ps^{1-\beta}$, with the factor
$Z_\alpha^{-\beta}$ absorbed into the normalizer $S$. This is why the intractable
$Z_\alpha$ never appears in the gradient even though it appears in the objective.

\subsection{Why the weights can be treated as constants}
\label{app:d3}

Equation~\eqref{eq:appgrad} contains $\bar\pi^{(\beta)}$, which depends on $\phi$
through $\ps^{1-\beta}$. It might appear that differentiating a loss whose
weights are computed from $\phi$ and then held fixed would miss a term. It does not,
because \eqref{eq:appgrad} is already the complete gradient. The differentiation in
Appendix~\ref{app:d2} was applied to the full expression for $\mathcal{J}$,
including the dependence of $\bar\pi^{(\beta)}$ on $\phi$, and every resulting term
was collected. In the right-hand side, $\bar\pi^{(\beta)}$ appears only as a
weight. Any implementation that evaluates those weights and holds them fixed while
differentiating $\log \ps$ therefore computes \eqref{eq:appgrad} exactly, not an
approximation of it. Appendix~\ref{app:verify} confirms this numerically against
brute-force enumeration.

\subsection{From an expectation to a weighted average}

We cannot evaluate the sum in \eqref{eq:appgrad}, because it runs over every
possible answer, and we cannot draw samples from $\bar\pi^{(\beta)}$, because
computing its normalizer $S$ requires the same sum. Instead, we draw answers from a
proposal distribution $q$ and correct for the mismatch.

Let $y^{(1)}, \dots, y^{(N)}$ be answers drawn from $q$, and suppose $q(y) > 0$
wherever $\bar\pi^{(\beta)}(y) > 0$. For any function $f$,
\[
\mathbb{E}_{\bar\pi^{(\beta)}}[f(y)]
\;=\; \sum_y \bar\pi^{(\beta)}(y)\, f(y)
\;=\; \sum_y q(y)\, \frac{\bar\pi^{(\beta)}(y)}{q(y)}\, f(y)
\;=\; \mathbb{E}_{q}\!\left[\frac{\bar\pi^{(\beta)}(y)}{q(y)}\, f(y)\right],
\]
where the middle step multiplies and divides by $q(y)$, which is legitimate
because $q(y) > 0$ on the support. The ratio $\bar\pi^{(\beta)}(y)/q(y)$ is the
importance weight of $y$ for estimating expectations under $\bar\pi^{(\beta)}$. We
call its self-normalized version the training weight, to keep it distinct from the
sampler's importance weight for $\pi_\alpha$.

Two problems remain: the right-hand side is still an expectation, which we replace
by the sample average over the $N$ answers, and $\bar\pi^{(\beta)}$ is known only
up to the constant $S$, which we cannot compute. The \emph{self-normalized} form
handles both. Write
$\tilde A^{(i)} = \bar\pi^{(\beta)}(y^{(i)})/q(y^{(i)})$ for the unnormalized
ratio, using any convenient constant multiple of $\bar\pi^{(\beta)}$, and let
\[
\bar A^{(i)} \;=\; \frac{\tilde A^{(i)}}{\sum_{j=1}^{N} \tilde A^{(j)}},
\qquad\text{so that}\qquad \sum_{i=1}^{N} \bar A^{(i)} = 1 .
\]
Any constant common to all $\tilde A^{(i)}$ cancels between the numerator and the
denominator, so the unknown $S$ and the unknown $Z_\alpha^{-\beta}$ both drop
out. The self-normalized estimate is
\[
\mathbb{E}_{\bar\pi^{(\beta)}}[f(y)] \;\approx\; \sum_{i=1}^{N} \bar A^{(i)} f\big(y^{(i)}\big),
\]
which is consistent, with bias of order $1/N$ \citep{doucet2001smc}. This is the
sense in which Section~\ref{sec:loss} obtains the training weights from
self-normalized importance sampling: they are the ratios $\bar\pi^{(\beta)}/q$
evaluated at the sampled answers and rescaled to sum to one. Taking $f = \nabla_\phi \log \ps$ and substituting into
\eqref{eq:appgrad},
\begin{equation}
\nabla_\phi D_\beta(\pi_\alpha \Vert \ps)
\;\approx\; -\sum_{i=1}^{N} \bar A^{(i)}\, \nabla_\phi \log \ps\big(y^{(i)}\big).
\label{eq:appmc}
\end{equation}

\subsection{Computing the weights token by token}

Both $\bar\pi^{(\beta)}$ and $q$ factor over tokens, so their ratio does too.
Writing the factorization out,
\[
\tilde A^{(i)}
\;\propto\; \frac{\prod_t \pt\big(y^{(i)}_t \mid x, y^{(i)}_{<t}\big)^{\alpha\beta}
\prod_t \ps\big(y^{(i)}_t \mid x, y^{(i)}_{<t}\big)^{1-\beta}}
{\prod_t q_t\big(y^{(i)}_t \mid x, y^{(i)}_{<t}\big)} ,
\]
and taking logarithms turns each product into a sum:
\[
\log \tilde A^{(i)} = \sum_t \Big[\alpha\beta \log \pt\big(y^{(i)}_t \mid \cdot\big)
+ (1-\beta)\log \ps\big(y^{(i)}_t \mid \cdot\big)
- \log q_t\big(y^{(i)}_t \mid \cdot\big)\Big],
\]
which is \eqref{eq:lossweight}. Working in logarithms is necessary, because each
factor is a probability below one and the product underflows to zero in floating
point after a few hundred tokens. The three quantities in the bracket are the same
three that the sampler evaluates at every token, which is the sense in which the
sampler and the loss share their computation.

\subsection{From the gradient to a scalar loss}

Equation~\eqref{eq:appmc} is a gradient estimate, but automatic differentiation
requires a scalar loss. Consider
\[
\mathcal{L}_{\mathrm{seq}}(\phi)
\;=\; -\sum_{i=1}^{N} \big[\bar A^{(i)}\big]_{\mathrm{sg}}\,
\log \ps\big(y^{(i)} \mid x\big),
\]
where $[\cdot]_{\mathrm{sg}}$ means the enclosed value is treated as a constant by
the differentiation. Differentiating, the constants pass through and
\[
\nabla_\phi \mathcal{L}_{\mathrm{seq}}
\;=\; -\sum_{i=1}^{N} \bar A^{(i)}\, \nabla_\phi \log \ps\big(y^{(i)} \mid x\big),
\]
which is exactly the right-hand side of \eqref{eq:appmc}. A gradient step on
$\mathcal{L}_{\mathrm{seq}}$ is therefore a gradient step on
$D_\beta(\pi_\alpha \Vert \ps)$, up to the error of the estimate.
Appendix~\ref{app:d3} justifies the stop-gradient.

Finally, we expand $\log \ps(y^{(i)} \mid x) = \sum_t \log \ps(y^{(i)}_t \mid x,
y^{(i)}_{<t})$, so that the loss is a sum over tokens. Dividing by
$\bar T = \frac1N\sum_i T_i$ keeps the magnitude of the loss from growing with
answer length, which would otherwise tie the effective learning rate to how long the
answers are. The result is the sequence-level term of \eqref{eq:loss}.

\section{The simplest case of the loss}
\label{app:simple}

Setting $\beta=1$ reduces \eqref{eq:lossweight} to a sum of the log incremental
importance weights of \eqref{eq:incw}. Substituting the proposal $q_t \propto \ps^{\alpha}$ and using
$\log Z^{\phi}_t = (1-\alpha)H_\alpha(\ps)$,
\[
\log \tilde A^{(i)} =
\alpha\Big[\log \pt(y^{(i)}\mid x) - \log \ps(y^{(i)}\mid x)\Big]
- (\alpha-1)\,\mathcal{H}^{(i)},
\qquad
\mathcal{H}^{(i)} = \sum_t H_\alpha\big(\ps(\cdot \mid x, y^{(i)}_{<t})\big).
\]
At initialization in self-distillation the first bracket vanishes and
$\bar A^{(i)} = \mathrm{softmax}_i\!\left(-(\alpha-1)\mathcal{H}^{(i)}\right)$, a
softmax over the negative total R\'enyi entropy along each answer, scaled by
$\alpha-1$. With $\lambda=0$ the loss is then maximum likelihood on the student's
own answers, each weighted by how confident the student was along it. As training
proceeds, the first bracket becomes nonzero and shifts weight toward answers that the
frozen teacher rates more highly than the current student does.

\section{Full experimental configuration}
\label{app:setup}

\paragraph{Models.} \texttt{Qwen2.5-Math-7B} (main, self-distillation),
\texttt{Qwen3-8B-Base}, \texttt{Qwen3-14B-Base}, \texttt{Llama-3.1-8B-Instruct},
\texttt{deepseek-math-7b-rl}, and
\texttt{Qwen2.5-Math-1.5B} as student with \texttt{Qwen2.5-Math-7B} as teacher for
the $\lambda$ sweep. In self-distillation teacher and student are initialized from
the same checkpoint; the teacher is then frozen.

\paragraph{Sampler.} $\alpha=4$, increased linearly from $1$ over the first $20\%$
of training steps; $N=16$ candidates; block length $\Delta=64$; resampling when
$\mathrm{ESS} < N/2$; systematic resampling; $\beta = 0.7$; the proposal
distribution is the student at temperature $1/\alpha$. The teacher's next-token
distributions are truncated to the top $64$ tokens and renormalized. The probability
retained by the truncation is logged, and it exceeds $0.98$ throughout except on the
first block of unusually long prompts.

\paragraph{Optimization.} $300$ steps, two prompts per step, learning rate
$5\times10^{-6}$ ($3\times10^{-6}$ for \texttt{Qwen3-14B-Base} and $1\times10^{-5}$ for the
\texttt{Qwen2.5-Math-1.5B} students) with $10$ warmup steps, no weight decay, gradient clipping at
$1.0$, 8-bit AdamW \citep{dettmers2022bit}, gradient checkpointing, microbatch
size $1$ in the gradient phase, $1024$ generated tokens during training ($2048$ for
\texttt{deepseek-math-7b-rl}, whose answers are longer).

\paragraph{Evaluation.} Four samples per problem, $200$ problems each from GSM8K
and MATH500, $2048$ generated tokens ($3072$ for
\texttt{Llama-3.1-8B-Instruct} and $4096$ for \texttt{deepseek-math-7b-rl}, whose
answers are longer). We use the prompt
builder, answer parser and grader of the reference power-sampling implementation
so that our numbers are directly comparable with published measurements on the
same models. Temperature-$1$ rows use each model's own generation configuration,
which is temperature $1$ with no truncation for the Qwen models and temperature
$0.6$ with nucleus sampling at $p=0.9$ for \texttt{Llama-3.1-8B-Instruct}.
Temperature-$1/\alpha$ rows set $\tau = 0.25$ explicitly and keep the rest of that
configuration, so \texttt{Llama-3.1-8B-Instruct} is sampled there with nucleus
sampling at $p=0.9$.

\paragraph{Prompt formats.} Each model is trained and evaluated with the same prompt
format. \texttt{deepseek-math-7b-rl} uses the instruction its model card specifies.

\paragraph{Training corpora.} GSM8K training-split questions throughout, except
Appendix~\ref{sec:openr1}, which uses $4000$ problem statements from
\texttt{OpenR1-Math-220k}. In both cases only the question text is used. For the
latter we remove any training prompt whose normalized form (lowercased,
alphanumeric characters only, first $400$ characters) matches any evaluation
problem, so no benchmark question appears in training.

\section{Validation of the evaluation code}
\label{app:harness}

Before training we evaluate the untrained model and compare with published
numbers for the same model, to establish that our decoding, parser and generation
length reproduce the setting those numbers were produced in.

\begin{center}\small
\begin{tabular}{lccc}
\toprule
untrained \texttt{Qwen2.5-Math-7B} & measured & published & difference \\
\midrule
MATH500, temperature $1$ & $50.25$ & $49.6$ & $+0.65$ \\
GSM8K, temperature $1$ & $59.25$ & $62.2$ & $-2.95$ \\
MATH500, temperature $1/\alpha$ & $70.13$ & $69.0$ & $+1.13$ \\
GSM8K, temperature $1/\alpha$ & $76.13$ & $77.3$ & $-1.17$ \\
\bottomrule
\end{tabular}
\end{center}

Differences are within $3$ points and run in both directions, consistent with
sampling variation on a $200$-problem subset rather than a systematic difference
in protocol. All comparisons in the paper use this evaluation.

\section{Complete accuracy tables}
\label{app:fullacc}

Table~\ref{tab:repl} reports the replication models at each model's default
generation configuration. The tables below give both decoding temperatures.

\begin{table}[h]
\caption{\texttt{Qwen3-8B-Base}, both decoding temperatures.}
\centering\small
\begin{tabular}{lcccc}
\toprule
 & \multicolumn{2}{c}{MATH500} & \multicolumn{2}{c}{GSM8K} \\
\cmidrule(lr){2-3}\cmidrule(lr){4-5}
 & pass@1 & best-of-4 & pass@1 & best-of-4 \\
\midrule
\multicolumn{5}{l}{\emph{temperature $1$}}\\
untrained & $65.50$ & $86.0$ & $78.00$ & $96.5$ \\
on-policy distillation & $75.00$ & $86.5$ & $93.00$ & $97.5$ \\
\textsc{oppd}, $\lambda=0$ & $\mathbf{78.25}$ & $\mathbf{87.5}$ & $\mathbf{95.13}$ & $\mathbf{99.0}$ \\
\midrule
\multicolumn{5}{l}{\emph{temperature $1/\alpha$}}\\
untrained & $78.50$ & $87.0$ & $93.75$ & $98.0$ \\
on-policy distillation & $79.00$ & $86.5$ & $95.00$ & $98.0$ \\
\textsc{oppd}, $\lambda=0$ & $\mathbf{81.00}$ & $86.5$ & $94.88$ & $97.0$ \\
\bottomrule
\end{tabular}
\end{table}

\begin{table}[h]
\caption{\texttt{Llama-3.1-8B-Instruct}. Lowering the decoding temperature gains
only $+0.6$ and $+1.25$ on this model, so improvements measured at its default
configuration (temperature $0.6$, nucleus sampling at $p=0.9$) come from training
rather than from decoding.}
\centering\small
\begin{tabular}{lcccc}
\toprule
 & \multicolumn{2}{c}{MATH500} & \multicolumn{2}{c}{GSM8K} \\
\cmidrule(lr){2-3}\cmidrule(lr){4-5}
 & pass@1 & best-of-4 & pass@1 & best-of-4 \\
\midrule
untrained, default configuration & $49.25$ & $69.5$ & $82.25$ & $92.5$ \\
untrained, temperature $1/\alpha$ & $49.88$ & $66.5$ & $83.50$ & $93.0$ \\
\midrule
on-policy distillation & $49.75$ & $65.5$ & $85.20$ & $94.0$ \\
\textsc{oppd}, $\lambda=0$ & $47.75$ & $63.0$ & $85.75$ & $91.5$ \\
\textsc{oppd}, $\lambda=1$ & $\mathbf{52.50}$ & $68.5$ & $85.38$ & $93.0$ \\
\textsc{oppd}, $\lambda=1$, OpenR1 prompts & $51.13$ & $\mathbf{69.0}$ & $\mathbf{85.88}$ & $\mathbf{94.0}$ \\
\bottomrule
\end{tabular}
\end{table}

\section{Complete divergence and absorbed-exponent measurements}
\label{app:fullmech}

\paragraph{Divergence to the power distribution for different numbers of
candidates.} Appendix~\ref{sec:replication} summarizes these measurements, and the
full table follows. The divergence $\hat D_2 = -\log(\mathrm{ESS}/N)$ is not
comparable across rows, because more candidates explore more of the distribution
and the measured spread of the weights grows with $N$ for any fixed proposal.
Comparisons are meaningful only within a row.

\begin{center}\small
\begin{tabular}{lccccc}
\toprule
& \multicolumn{2}{c}{$\mathrm{ESS}/N$} & \multicolumn{2}{c}{$\hat D_2$ per block} & \\
candidates $N$ & untrained & \textsc{oppd} & untrained & \textsc{oppd} & reduction \\
\midrule
$4$ & $0.610$ & $\mathbf{0.696}$ & $0.494$ & $\mathbf{0.362}$ & $26.6\%$ \\
$16$ & $0.490$ & $\mathbf{0.585}$ & $0.712$ & $\mathbf{0.536}$ & $24.7\%$ \\
$64$ & $0.427$ & $\mathbf{0.475}$ & $0.851$ & $\mathbf{0.744}$ & $12.6\%$ \\
\bottomrule
\end{tabular}
\end{center}

Because the ESS measures the divergence from the power distribution to the proposal
distribution, the trained model is also a more efficient proposal for the sampler
than the model it was trained from. Fewer candidates are therefore wasted if the
sampler is run after training.

\paragraph{Absorbed exponent on both sets of contexts, \texttt{Qwen2.5-Math-1.5B}
sweep.} Table~\ref{tab:dose} reports the measurements on each model's own answers.
Both sets of contexts are given here, with the token entropies. The base entropies
differ between rows because each row's contexts are written by a different model,
so only the fraction of the base-to-target entropy gap that training closes is
comparable across rows.

\begin{center}\small
\begin{tabular}{lccccc}
\toprule
 & \multicolumn{2}{c}{$\aeff(\pb\to\ps)$} & \multicolumn{3}{c}{entropy on own contexts (nats)} \\
\cmidrule(lr){2-3}\cmidrule(lr){4-6}
 & base ctx. & own ctx. & base & trained & target \\
\midrule
\textsc{oppd}, $\lambda=0$ & $1.207$ & $\mathbf{2.019}$ & $0.268$ & $0.087$ & $0.052$ \\
\textsc{oppd}, $\lambda=0.5$ & --- & $1.283$ & $0.218$ & $0.138$ & $0.045$ \\
\textsc{oppd}, $\lambda=1$ & $1.165$ & $1.243$ & $0.277$ & $0.203$ & $0.056$ \\
\textsc{oppd}, $\lambda=2$ & $1.137$ & $1.193$ & $0.300$ & $0.224$ & $0.062$ \\
baseline & $1.105$ & $1.144$ & $0.302$ & $0.241$ & $0.061$ \\
\bottomrule
\end{tabular}
\end{center}

\paragraph{Validation of the estimator.} On synthetic logits with a known exponent,
the estimator returns $a$ correctly to four decimal places for
$a \in \{1, 1.5, 2, 3, 4\}$. It returns $1.998$ when the sharpened model is also
perturbed by noise, and $-0.008$ for a model unrelated to the reference.

\paragraph{A resolution limit of the estimator.} The fitted slope in
\eqref{eq:aeff} is a covariance divided by the variance of $\log\pb$ under
weights $\pb$. When $\pb$ is close to deterministic, that variance approaches zero
and the slope is dominated by noise. For the $7$B and $8$B models, the entropy of
$\pb$ along the trained models' own answers is $0.096$ to $0.138$ nats, and all
trained models give $\aeff$ between $1.11$ and $1.25$, with no separation between
them. We therefore report $\aeff$ only where the entropy of $\pb$ is large enough
($0.22$ to $0.30$ nats in the \texttt{Qwen2.5-Math-1.5B} sweep), and use the direct
entropy measurement of Table~\ref{tab:mech} for the larger models. Outside this
range, values of $\aeff$ mostly reflect noise.

\section{Why the student at temperature 1 fails as a proposal}
\label{app:proposal}

Section~\ref{sec:temp} states that using the student at temperature $1$ as the
proposal distribution makes the sampler fail. The second term of
\eqref{eq:decomp}, $(\alpha-1)\log\ps(y_t)$, has a standard deviation of about
$4.6$ nats per token at $\alpha=4$. Across a block of $\Delta=64$ tokens, the
spread between candidates is therefore about $4.6\sqrt{64} \approx 37$ nats, against
a usable range of $\log N = \log 16 \approx 2.8$ nats. With this proposal,
$\mathrm{ESS}/N$ stays at $1/N$ from the first block onward, which makes the sampler
equivalent to selecting one candidate at random. The failure raises no error and
produces no unusual value in the training log. The normalized importance weights are
a softmax over the log importance weights, so a gap of $35$ nats between the two
largest gives the second a relative weight of $e^{-35} \approx 6\times10^{-16}$.
Resampling then replaces all $N$ candidates with copies of one. From that point on,
the candidates carry a single distinct answer and no longer approximate $\pi_\alpha$,
and the estimator of Theorem~\ref{thm:master} has a single sample. With the student
at temperature $1/\alpha$ as the proposal, this term cancels exactly, and the
measured $\mathrm{ESS}/N$ during training lies between $0.39$ and $0.64$ for every
model we train.

As a check on the algebra, setting the student equal to the teacher makes the
teacher--student difference term vanish, so by Theorem~\ref{thm:opt} the
incremental importance weights should be identical for every candidate. The measured
$\mathrm{ESS}/N$ in this configuration is $0.998$.

\section{Implementation correctness}
\label{app:verify}

Several results concern distributions rather than accuracy, so we verify the
implementation independently of any benchmark with fifty automated checks on small
randomly initialized models, which run the same code as training.

\begin{itemize}\itemsep2pt
\item Theorem~\ref{thm:master}, and the exactness of holding the weights fixed,
against brute-force enumeration and numerical differentiation: agreement to
$10^{-10}$ for $\beta \in \{0.25, 0.5, 1, 2, 5\}$.
\item Log-probabilities recorded during generation and recomputed in the gradient
step agree to $4.8\times10^{-7}$ with resampling forced at every block. This tests
token alignment and the reordering of the attention key-value cache for student and
teacher.
\item Loss and gradient match a direct reference implementation to $10^{-6}$ and
$3\times10^{-8}$, and accumulating gradients over microbatches matches a single
backward pass to $7\times10^{-9}$.
\item The total variation distance between the sampler's output distribution and
$\pi_\alpha$, computed by exhaustive enumeration on a small model, decreases
monotonically in $N$ ($0.166$, $0.095$, $0.065$ at $N=4,16,64$) against $0.112$ for
the untrained model. This is the $O(1/N)$ behavior expected of a correct sampler.
\item With the student equal to the teacher, Theorem~\ref{thm:opt} predicts equal
weights, and the measured $\mathrm{ESS}/N$ is $0.998$ against a theoretical $1.000$.
\end{itemize}

These checks found three defects before any reported run: the amplification of
training weights at resampling (Section~\ref{sec:impl}), activation memory that grew
with $N$ instead of with the microbatch size, and proposal log-probabilities that
were discarded.

\section{Correcting the weight advantage of stopping early}
\label{app:brevity}

Equation~\eqref{eq:renyi} implies that the log importance weight decreases at every
token. A candidate that emits the end-of-sequence token keeps its weight fixed while
unfinished candidates continue to lose weight, so stopping early gains weight for
reasons unrelated to answer quality. For long chains of reasoning, this would train
the model to stop reasoning early, and the loss curve would give no sign of it.

We charge each finished candidate the mean per-token increment in log importance
weight of the unfinished ones. This is the same constant for every finished
candidate at a given step, so the normalized weights among unfinished candidates are
unchanged, and only the comparison between finished and unfinished candidates is
corrected. As a diagnostic, we use the mean answer length weighted by the normalized
importance weights minus the unweighted mean length, which should be near zero.

We measure this diagnostic in a controlled setting where student and teacher are the
same model, so that the per-token decrease is identical across candidates and any
preference for length comes from the weight accounting:

\begin{center}\small
\begin{tabular}{lcc}
\toprule
 & weighted minus unweighted length & $\mathrm{ESS}/N$ \\
\midrule
without the correction & $-25.9$ tokens & $0.084$ \\
with the correction & $+0.13$ tokens & $0.997$ \\
\bottomrule
\end{tabular}
\end{center}

\section{Resampling and the training weights}
\label{app:resample}

Resampling duplicates candidates. The training weights already recorded for earlier
blocks are stored per candidate, so a duplicated candidate carries its ancestor's
recorded weight once per copy. Without a correction, the total training weight on an
earlier block grows with every resampling step (a block that survives $k$ steps is
amplified by up to $N^{k}$), and the opening tokens of every answer come to dominate
the loss. Neither the loss value nor any standard diagnostic shows this.

Duplication is how resampling represents weight, so the recorded training weights
must be renormalized to compensate. After each resampling step, we renormalize the
stored training weights of every earlier block to sum to one across candidates. The
test suite of Appendix~\ref{app:verify} checks this property automatically.

The sampler and the loss also keep separate weight accumulators. The sampler's
accumulator uses \eqref{eq:incw} and decides when to resample. Replacing it with the
training weight, which depends on $\beta$, would stop the sampler from approximating
$\pi_\alpha$. The loss accumulator uses \eqref{eq:lossweight}, is reset at every
block boundary, and supplies $\bar A^{(i)}$. Their per-token increments agree at
$\beta=1$, but for every $\beta$ they accumulate over different windows. We write
the training weight with the letter $A$ to keep it distinct from the importance
weight $W$.

\section{Selecting among finished answers}
\label{app:selection}

\begin{table}[h]
\caption{Percentage of candidates that are correct, against the total normalized
importance weight on correct candidates. If the teacher's probabilities ranked
finished answers usefully, the weight on correct candidates would exceed the
percentage correct.}
\centering\small
\begin{tabular}{llccc}
\toprule
model & proposal, $N$ & \% correct & weight on correct & difference \\
\midrule
\multirow{2}{*}{\texttt{1.5B}} & untrained, $16$ & $53.5$ & $53.0$ & $-0.45$ \\
 & trained, $16$ & $49.8$ & $49.7$ & $-0.06$ \\
\midrule
\multirow{3}{*}{\texttt{Qwen2.5-Math-7B}} & untrained, $16$ & $78.9$ & $78.8$ & $-0.10$ \\
 & \textsc{oppd} $\lambda{=}1$, $16$ & $76.4$ & $76.7$ & $+0.32$ \\
 & \textsc{oppd} $\lambda{=}0$, $16$ & $80.0$ & $80.1$ & $+0.15$ \\
\midrule
\multirow{6}{*}{\texttt{Qwen3-8B-Base}} & untrained, $4$ & $81.20$ & $81.04$ & $-0.16$ \\
 & untrained, $16$ & $82.40$ & $82.34$ & $-0.06$ \\
 & untrained, $64$ & $76.72$ & $76.63$ & $-0.09$ \\
 & \textsc{oppd}, $4$ & $81.30$ & $81.46$ & $+0.16$ \\
 & \textsc{oppd}, $16$ & $80.10$ & $80.18$ & $+0.08$ \\
 & \textsc{oppd}, $64$ & $80.06$ & $79.86$ & $-0.20$ \\
\bottomrule
\end{tabular}
\end{table}

\end{document}